\documentclass[11pt]{article}

\usepackage[final]{acl}

\usepackage{times}
\usepackage{latexsym}
\usepackage{microtype}
\usepackage{graphicx}
\usepackage{booktabs} 
\usepackage{enumitem}
\usepackage{tabularx}
\usepackage{multirow}
\usepackage{tcolorbox}
\usepackage{adjustbox}
\usepackage{subcaption}
\usepackage[utf8]{inputenc}
\usepackage{makecell}
\usepackage{listings}
\usepackage[table]{xcolor} 
\usepackage{booktabs}      
\usepackage{CJKutf8}
\usepackage{amsmath}
\usepackage{amssymb}
\usepackage{mathtools}
\usepackage{amsthm}

\usepackage[capitalize,noabbrev]{cleveref}
\usepackage{graphicx}
\usepackage{dsfont}
\usepackage{amsfonts}
\usepackage{amsmath}
\usepackage{stackrel}
\usepackage{hyperref}
\usepackage[framemethod=tikz]{mdframed}
\usepackage{xcolor}
\usepackage{tikz}
\usepackage{tabularx}

\usepackage{tcolorbox}
\tcbuselibrary{breakable, skins, theorems}
\usepackage{adjustbox}
\usepackage[framemethod=tikz]{mdframed}
\newenvironment{maintheorem}[1][]{%
  \begin{tcolorbox}[
    colback = blue!5,
    colframe = white,
    fonttitle = \bfseries,
    breakable = true]
  \begin{theorem}[#1]%
}{%
  \end{theorem}%
  \end{tcolorbox}
}

\usetikzlibrary{arrows.meta,positioning,fit,calc,backgrounds}
\definecolor{PromptBorder}{RGB}{180,180,180}
\definecolor{PromptBg}{RGB}{248,248,248}
\definecolor{UserColor}{RGB}{0,92,185}
\definecolor{AsstColor}{RGB}{0,128,100}
\definecolor{SysColor}{RGB}{120,70,160}

\mdfdefinestyle{promptstyle}{
  backgroundcolor=PromptBg,
  linecolor=PromptBorder,
  linewidth=0.6pt,
  roundcorner=6pt,
  innertopmargin=8pt,
  innerbottommargin=8pt,
  innerleftmargin=10pt,
  innerrightmargin=10pt,
  skipabove=6pt,
  skipbelow=6pt,
  splittopskip=\baselineskip,
  splitbottomskip=\baselineskip,
  nobreak=false,
}

\newcommand{\prompttitle}[1]{%
  \textbf{#1}\par\vspace{0.3em}\hrule\vspace{0.6em}
}

\newcommand{\usr}{\textcolor{UserColor}{\textbf{User}}}
\newcommand{\asst}{\textcolor{AsstColor}{\textbf{Assistant}}}

\newcommand{\chatline}[2]{%
  \noindent\textbf{#1:}\hspace{0.5em}#2\par
}
\newcommand{\JPEN}[2]{#1\\{\footnotesize\emph{#2}}\\}
\newenvironment{promptbox}[1]{%
  \begin{mdframed}[style=promptstyle]
  \prompttitle{#1}
}{%
  \end{mdframed}
}

\newcommand{\BOS}{\textsc{BOS}}
\newcommand{\EOS}{\textsc{EOS}}

\newcommand{\bv}{\boldsymbol{v}}

\DeclareMathOperator*{\argmax}{arg\,max}
\newcommand{\E}{\mathbb{E}}

\theoremstyle{plain}
\newtheorem{theorem}{Theorem}[section]
\newtheorem{proposition}[theorem]{Proposition}
\newtheorem{lemma}[theorem]{Lemma}

\theoremstyle{definition}
\newtheorem{definition}[theorem]{Definition}
\newtheorem{assumption}[theorem]{Assumption}
\theoremstyle{remark}
\newtheorem{remark}[theorem]{Remark}
\usepackage{hyperref}
\usepackage[framemethod=tikz]{mdframed}
\usepackage{xcolor}
\usepackage{tikz}
\usepackage{tabularx}

\usepackage{tcolorbox}
\tcbuselibrary{breakable, skins, theorems}
\usepackage[T1]{fontenc}
\usepackage[framemethod=tikz]{mdframed}
\usepackage{tikz}
\usepackage[utf8]{inputenc}
\usepackage{amsmath}
\usepackage{amssymb}
\usepackage{mathtools}
\usepackage{amsthm}
\usepackage{adjustbox}
\usepackage[framemethod=tikz]{mdframed}
\usepackage[capitalize,noabbrev]{cleveref}
\usepackage{graphicx}
\usepackage{dsfont}
\usepackage{amsfonts}
\usepackage{amsmath}
\usepackage{stackrel}
\usepackage{microtype}

\usetikzlibrary{arrows.meta,positioning,fit,calc,backgrounds}
\definecolor{PromptBorder}{RGB}{180,180,180}
\definecolor{PromptBg}{RGB}{248,248,248}
\definecolor{UserColor}{RGB}{0,92,185}
\definecolor{AsstColor}{RGB}{0,128,100}
\definecolor{SysColor}{RGB}{120,70,160}

\usepackage{inconsolata}

\usepackage{graphicx}

\title{Consensus Group Relative Policy Optimization for Text Generation}

\author{
  Yuki Ichihara$^{1}$,
  Yuu Jinnai$^{2}$,
  Kaito Ariu$^{2}$,
  Eiji Uchibe$^{3}$ \\[1ex]
  $^{1}$Nara Institute of Science and Technology \quad
  $^{2}$CyberAgent \\
  $^{3}$Advanced Telecommunications Research Institute International 
}

\begin{document}
\maketitle
\begin{abstract}

Many strong decoding methods for text generation follow a sample-and-rerank paradigm: they draw multiple candidates, score each under a utility (reward) function using consensus across samples, and choose the best one. Although effective, these methods incur high computational costs during inference due to repeated sampling and scoring.
Prior attempts to amortize inference-time computation typically rely on gold references, teacher labels, or curated preference data, thereby increasing the effort required to construct datasets and the demand for high-fidelity reward models.
We propose Consensus Group Relative Policy Optimization (C-GRPO), which distills Minimum Bayes Risk (MBR) decoding into training by formulating the consensus utility as a group-relative objective within GRPO. C-GRPO requires only a utility function and policy samples, without gold references or explicit preference labels.
As supporting analysis, we study an idealized sentence-level surrogate and show that the expected update is directionally consistent with the expected-utility objective underlying MBR decoding under simplifying assumptions. 
Experiments on machine translation (WMT 2024), text summarization (XSum), and Japanese social-bias mitigation (JBBQ) show that C-GRPO often matches or exceeds inference-time MBR decoding while replacing multi-sample reranking with a single generation. 
\end{abstract}

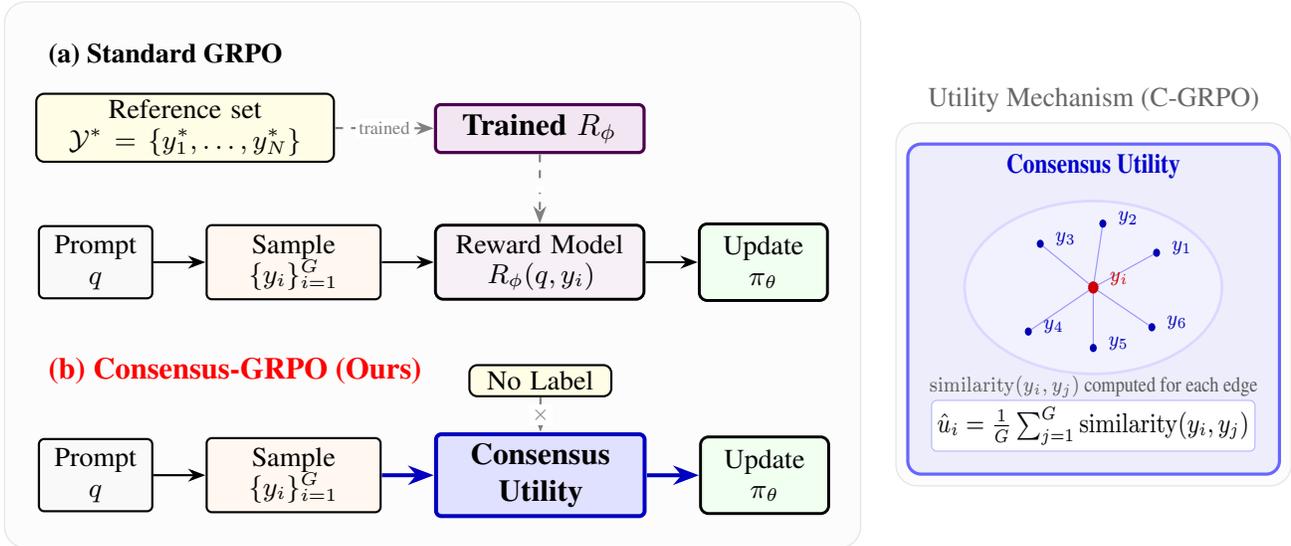
\begin{figure*}[t]
\centering
\begin{tikzpicture}[
  scale=1, transform shape,
  font=\small,
  node distance=5mm and 6mm,
  io/.style={draw, thick, fill=gray!5, rounded corners=2pt, inner sep=3pt, align=center, text width=1.1cm},
  train_op/.style={draw, thick, fill=orange!6, rounded corners=2pt, inner sep=3pt, align=center, text width=1.8cm},
  model/.style={draw, line width=1.1pt, fill=green!6, rounded corners=3pt, inner sep=4pt, align=center, text width=1.2cm},
  rm/.style={draw, line width=1.1pt, fill=violet!6, rounded corners=3pt, inner sep=3pt, align=center, text width=2.2cm},
  data/.style={draw, thick, fill=yellow!12, rounded corners=3pt, inner sep=2pt, align=center, text width=1.6cm},
  key_op/.style={train_op, draw=blue!75!black, fill=blue!12, line width=1.5pt, font=\bfseries, text width=2.2cm},
  rm_train/.style={train_op, draw=violet!60!black, fill=violet!10, line width=1.2pt, font=\bfseries, text width=2.2cm},
  arr/.style={-Stealth, thick},
  arr_key/.style={-Stealth, line width=1.5pt, draw=blue!75!black},
  bg_panel/.style={draw=gray!20, fill=gray!2, rounded corners=8pt, line width=0.5pt},
  rew_detail_box/.style={draw=#1, fill=#1, fill opacity=0.10, line width=1.2pt, rounded corners=5pt, text opacity=1},
  header/.style={font=\bfseries, color=gray!80!black}
]

\begin{scope}[local bounding box=pipelines]
  \node[io] (qA) {Prompt\\$q$};
  \node[train_op, right=of qA] (sA) {Sample\\$\{y_i\}^G_{i=1}$};
  \node[rm, right=of sA] (rA) {Reward Model\\$R_\phi(q, y_i)$};
  \node[model, right=of rA] (piA) {Update\\$\pi_\theta$};

  \draw[arr] (qA) -- (sA);
  \draw[arr] (sA) -- (rA);
  \draw[arr] (rA) -- (piA);

  \node[data, above=7mm of rA, xshift=-41mm, text width=3.3cm] (refsetA)
    {Reference set\\$\mathcal{Y}^*=\{y^*_1,\dots,y^*_N\}$};

  \node[rm_train, above=8mm of rA, xshift=0mm] (trainRM)
    {Trained $R_\phi$};

  \draw[arr, dashed, gray] (refsetA) -- node[midway, fill=white, inner sep=1pt, font=\tiny]{trained}(trainRM);

  \draw[arr, dashed, gray] (trainRM) -- node[midway, fill=white, inner sep=1pt, font=\tiny]{} (rA);


  \node[anchor=south west, font=\bfseries\footnotesize]
    at ([yshift=17mm]qA.north west) {(a) Standard GRPO};

  \node[io, below=1.6cm of qA] (qB) {Prompt\\$q$};
  \node[train_op, right=of qB] (sB) {Sample\\$\{y_i\}^G_{i=1}$};
  \node[key_op, right=of sB] (uB) {Consensus Utility};
  \node[model, right=of uB] (piB) {Update\\$\pi_\theta$};

  \draw[arr] (qB) -- (sB);
  \draw[arr_key] (sB) -- (uB);
  \draw[arr_key] (uB) -- (piB);

  \node[data, above=4mm of uB, text width=1.5cm] (hfB) {No Label};
  \draw[arr, dashed, gray] (hfB) -- node[midway, fill=white, inner sep=1pt, font=\tiny]{$\times$} (uB);

  \node[anchor=south west, font=\bfseries, color=red]
    at ([yshift=5mm]qB.north west) {(b) Consensus-GRPO (Ours)};
\end{scope}

\begin{scope}[local bounding box=details, xscale=0.85]
  \coordinate (RightPanelCenter) at ($(pipelines.east) + (3.6, 0)$);
  \coordinate (MidRU) at ($(rA)!0.22!(uB)$);

  \node[minimum width=4.8cm, minimum height=3.6cm] (detB) at (RightPanelCenter |- MidRU) {};

  \begin{scope}[shift={($(detB.center)+(0,2mm)$)}]
    \node[font=\footnotesize\bfseries, blue!80!black] at (0, 1.45) {Consensus Utility};

    \node[
      font=\footnotesize,
      fill=white, draw=blue!30, rounded corners=2pt,
      inner sep=2.2pt
    ] at (0, -1.55)
    {$\hat{u}_i=\frac{1}{G}\sum_{j=1}^{G}\mathrm{similarity}(y_i,y_j)$};

    \begin{scope}[yshift=-1.5mm]
      \draw[blue!15, fill=blue!5, line width=1pt] (0,0.20) ellipse (1.75 and 1.00);

      \coordinate (Yi) at (0,0.20);

      \foreach \k/\ang/\dist in {
        1/25/0.95,
        2/80/0.75,
        3/145/0.88,
        4/210/1.02,
        5/270/0.70,
        6/330/0.92
      }{
        \coordinate (P\k) at ($(Yi)+(\ang:\dist)$);
        \draw[blue!40, -{Stealth[scale=0.45]}, line width=0.45pt] (P\k) -- (Yi);
        \fill[blue!70!black] (P\k) circle (1.35pt);
        \node[font=\scriptsize, color=blue!70!black, anchor=west]
          at ($(P\k)+(0.08,0.06)$) {$y_{\k}$};
      }
\fill[red!80!black] (Yi) circle (2.0pt);
\node[font=\scriptsize, color=red!80!black, anchor=west]
  at ($(Yi)+(0.10,0.10)$) {$y_i$};
      \node[font=\scriptsize, color=gray!70!black, align=left, anchor=west]
        at ($(Yi)+(-2.35,-1.15)$)
        {$\mathrm{similarity}(y_i,y_j)$ computed for each edge};
    \end{scope}
  \end{scope}
\end{scope}

\begin{scope}[on background layer]
  \node[bg_panel, fit=(pipelines), inner sep=12pt] (LPanel) {};
  \node[header, anchor=south, font=\small] at (LPanel.north) {Training Pipeline Overview};

  \node[bg_panel, fit=(details), inner sep=8pt] (RPanel) {};
  \node[header, anchor=south, font=\small] at (RPanel.north) {Utility Mechanism (C-GRPO)};

  \node[rew_detail_box=blue!60, fit=(detB)] {};
\end{scope}

\end{tikzpicture}
\caption{\textbf{System Overview.} (Left) Training pipelines for GRPO and C-GRPO. C-GRPO is different from standard GRPO because it does not use reward supervision, unlike standard GRPO, which typically uses a reward or preference model. (Right) C-GRPO’s utility mechanism is based on \textbf{self-consensus} among sampled candidates.}\label{fig:c-grpo}
\end{figure*}

\section{Introduction}
In recent years, many of the most effective methods for generating text can be described as consensus inference \cite{kumar-byrne-2002-minimum,kumar-byrne-2004-minimum,ehling-etal-2007-minimum,watanabe-sumita-2011-machine, eikema-aziz-2022-sampling}. 
Given an input, these methods (i) sample a set of candidate outputs, (ii) assign each candidate a score using consensus-based methods across the sampled set under a task-specific utility function, and (iii) output the candidate with the highest consensus score. This often yields better-quality results in practice \cite{ehling-etal-2007-minimum, eikema-aziz-2020-map, muller-sennrich-2021-understanding, eikema-aziz-2022-sampling, bertsch-etal-2023-mbr}.

The downside is that the consensus method is expensive at inference time \cite{watanabe-sumita-2011-machine, eikema-aziz-2022-sampling, wang2023selfconsistency, bertsch-etal-2023-mbr,cheng-vlachos-2023-faster,NEURIPS2024_57c89126}. 
It requires multiple forward generations per input, as well as additional computation to compare candidates and consensus scores for each prompt.
As models scale and applications become latency-sensitive, this cost creates a persistent quality–latency bottleneck: practitioners often face a choice between cheap single-pass decoding (lower quality) and expensive consensus procedures (higher quality).
To eliminate the inference-time overhead of consensus decoding, prior work distills the consensus decision rule into a model, typically requiring gold references or teacher labels, which in turn necessitate building task-specific datasets and/or the training of high-precision reward models \cite{finkelstein2024mbr,finkelstein-etal-2024-introducing,yang-etal-2024-direct}.

In this paper, we instead distill the consensus decision rule into a single-pass policy in a reference-free manner\footnote{Here, reference-free means that training does not require gold references or explicit preference labels; it uses only a task utility function together with on-policy samples.}.
Our key observation is that many consensus objectives are defined based on a sample of candidates and depend solely on relative scores within the sample (e.g., comparisons to the group average). 
This yields a signal of within-group advantage, in which candidates are encouraged or discouraged depending on whether they exceed the group's typical quality level.
This within-group structure directly motivates the use of Group Relative Policy Optimization (GRPO) \cite{shao2024deepseekmath}: GRPO updates the policy using a group-relative advantage, increasing the probability of candidates whose consensus scores exceed the group baseline.

Specifically, we instantiate this idea with Minimum Bayes Risk (MBR) decoding, where each candidate is scored by its average utility against other samples \cite{kumar-byrne-2002-minimum,kumar-byrne-2004-minimum,eikema-aziz-2022-sampling}.
Importantly, this consensus utility can be computed directly from on-policy samples using only a task utility function, eliminating the need for gold references or teacher-provided labels.

We propose \textbf{Consensus-GRPO (C-GRPO)}, a simple and general approach to distill expensive consensus-based decoding into a \emph{single-pass} policy trained with GRPO.
Unlike standard GRPO, which typically assumes a reward or preference model, C-GRPO does not rely on reward supervision (see Figure~\ref{fig:c-grpo}).
Instead, it derives its training signal by aligning with other sampled candidates under a utility function.
Consequently, C-GRPO does not require gold references, teacher labels, or an explicitly collected preference dataset. It also does not require training a separate reward model. The only necessary ingredient is a task utility function. The optimization signal is derived solely from within-group consensus.

\paragraph{Contributions.}
Our main contributions are the proposed reference-free distillation objective and its empirical validation:
\begin{itemize}[leftmargin=*, itemsep=2pt, topsep=2pt]
  \item \textbf{Objective: distilling consensus into one pass.}
  We introduce a GRPO training objective that constructs group-relative advantages from a consensus-based utility computed on policy samples, distilling consensus decoding into a single-pass policy at test time (requiring no gold references or explicit preference labels).
  \item \textbf{Experiments: quality without reranking.}
In machine translation and summarization, C-GRPO matches or exceeds inference-time MBR decoding across several settings while using a single test-time generation step rather than sample-and-rerank decoding. Furthermore, we find that C-GRPO learns from higher-quality outputs that are closer to gold references under reference-based evaluation, suggesting that training shifts probability mass toward near-reference generations.
We also demonstrate generalization across other model families and show that C-GRPO outperforms self-rewarding on the Japanese Bias Benchmark for QA, a sensitive social-bias mitigation setting.
  \item \textbf{Analysis: under ideal assumptions, C-GRPO is consistent.}
We analyze a sentence-level surrogate and show that, under several ideal assumptions, the expected update of C-GRPO is directionally consistent with the target consensus objective, with the rate of $\mathcal{O}(\frac{1}{\sqrt{T}})$. The analysis suggests that C-GRPO is a reasonable learning algorithm in some scenarios. 
\end{itemize}

\section{Background}
Text generation involves producing an output sequence based on an input sequence; the set of input sequences is defined by $\mathcal{X}$. Probabilistic text generators define a probability distribution over the output space of hypotheses $\mathcal{Y}$. The set of complete hypotheses $\mathcal{Y}$ is:
\begin{equation*}
    \mathcal{Y} := \{\BOS \circ \bv \circ \EOS | \bv \in \mathcal{V}^*\},
\end{equation*}
where $\circ$ is a string concatenation and $\mathcal{V}^*$ is the Kleene closure of a set of vocabulary $\mathcal{V}$. 
\subsection{Group Relative Policy Optimization}
Group Relative Policy Optimization (GRPO) is a reinforcement learning algorithm \cite{shao2024deepseekmath} and has proven effective across a range of domains\cite{ding2025multi,liu2025flow,xue2025dancegrpo,liu2025flow,xue2025dancegrpo}.
Let $q\in \mathcal{X}$ denote the initial state (prompt), and the policy $\pi_\theta\left(\cdot \mid q\right)$ outputs the action (sentence) $y \in \mathcal{Y}$ based on the initial state $q$ from the action space. Formally, let $R$ be the reward function, which is a mapping from a prompt-output pair to a scalar value.
For each prompt $q$, GRPO samples a group of outputs $ \mathcal{G} = \left\{y_1, y_2,...,y_G\right\}$ from the old policy $\pi_{\theta_{\text {old}}}$ and then optimizes the policy model by maximizing the following objective:
\begin{align}
\mathcal{J}_{\text{GRPO}}(\theta)
&= 
\E\Biggl[\sum_{i=1}^G \frac{1}{|y_i|}\frac{1}{G} 
         \frac{\pi_\theta\bigl(y_i\mid q\bigr)}%
              {\pi_{\theta_\mathrm{old}}\bigl(y_i\mid q\bigr)}
         A_{i}\nonumber\Biggr]\\
&- \beta \text{KL}(\pi_\theta, \pi_{\theta_{\text{ref}}})
, \label{eq:grpy_obj}
\end{align}
where $\beta$ is a hyperparameter, and $\text{KL}$ is Kullback–Leibler (KL) divergence.
In our experiments, we set $\beta=0$ \cite{liu2025understanding,shao2025spurious}; for clarity, we omit the KL term in the derivations.
For simplicity, we define the importance ratio $\rho_i(\theta) := \frac{\pi_\theta(y_i\mid q)}{\pi_{\mathrm{old}}(y_i\mid q)}$.
$A_i$ represents the normalized advantage value of the sentence $y_i$ using a reward function:
\begin{equation}\label{eq:advantage}
A_{i}= \frac{R(q,y_i)-\mathrm{mean}_\mathcal{G}\bigl( R(q,\mathcal{G})\bigr)}{\mathrm{std}_\mathcal{G}\bigl( R(q,\mathcal{G})\bigr)}.
\end{equation}
For simplicity, we present the unclipped form and omit PPO-style clipping (threshold $\epsilon$ and the $\min$ operator); we use the standard GRPO implementation in experiments.

We also consider Dr.GRPO \citep{liu2025understanding}, a variant of GRPO.
Following \citet{liu2025understanding}, we subtract the (within-group) standard deviation term and output length $|y_i|$ from each sample’s contribution in Eq.~\eqref{eq:grpy_obj} and from the advantage in Eq.~\eqref{eq:advantage}.

\subsection{Consensus-based Decoding}\label{sec:mbr}
An example of sample-and-consensus inference is \emph{expected-utility} consensus, where each candidate is scored based on its expected similarity (utility) to likely outputs under a reference distribution, and the candidate with the highest score is selected.
\textbf{Minimum Bayes Risk (MBR) decoding} is an example of one type of consensus-based decoding \citep{kumar-byrne-2002-minimum,kumar-byrne-2004-minimum}.
In neural text generation, consensus-based decoding is commonly instantiated by sampling candidates from a model distribution and
estimating expected utility via Monte Carlo, using task-specific similarity metrics
\citep{eikema-aziz-2022-sampling,farinhas-etal-2023-empirical,bertsch-etal-2023-mbr}.

The goal of decoding is to find the best hypothesis for a given input. 
Let $q \in \mathcal{X}$ be an input prompt and $y \in \mathcal{Y}$ be an output sequence.
Let $P_{\text{model}}(\cdot\mid q)$ be the policy model (e.g., Large Language Models (LLMs)).
We assume a bounded utility function $u:\mathcal{Y}\times\mathcal{Y}\to[0,1]$ measuring similarity between multiple texts.
The procedure of MBR decoding consists of two components: a text generation model $\pi_\theta$ and a utility metric $u(y, y')$.
The utility function $u(y, y')$ estimates the quality of a candidate output $y$.
MBR decoding selects the best hypothesis according to its expected utility over the text generation model's probability $\pi_\theta$:

\begin{definition}[Minimum Bayes Risk decoding]\label{def:consensus}
\begin{align}
    &u_\text{model}(y\mid q) = \sum_{y_i \in \mathcal{Y}} u(y, y_i)\cdot P_{\text{model}}(y_i\mid q),\label{eq:model-MBR decoding}\\
     &y^{\mathrm{model}} = \argmax_{y \in \mathcal{Y}} u_\text{model}(y\mid q),\nonumber
\end{align}
where $P_{\text{model}}$ is the distribution of language model. 
\end{definition}
$u_\text{model}(y\mid q)$ is the expected utility of a candidate $y$ under the model distribution, and $y^{\mathrm{model}}$ is the optimal output of MBR decoding under $\mathcal{Y}$.
Since integration over $\mathcal{Y}$ is computationally intractable, Eq.~\eqref{eq:model-MBR decoding} is approximated by a Monte Carlo estimate \cite{eikema-aziz-2022-sampling,farinhas-etal-2023-empirical} using a reference sample set $\mathcal{G}=\{y'_1,\dots,y'_G\}$ with $y'_j\sim P_{\text{model}}(\cdot\mid q)$
In practice, we instantiate both the hypothesis set and the reference set by the same finite candidate set $\mathcal{G}$:
\begin{align}
    &\widehat{u}(y\mid q) = \frac{1}{G} \sum^{G}_{i=1} u(y, y'_i),\label{eq:monte_MBR decoding}\\
    &y^{\mathrm{monte}} = \argmax_{y \in \mathcal{G}} \widehat{u}(y\mid q)\nonumber. 
\end{align}
Because of the intractability of $\mathcal{Y}$, $y^{\mathrm{monte}}$ is considered the optimal output for MBR decoding.

MBR decoding is appealing  because it directly optimizes a task-specific utility function, selecting outputs that are not merely
maximizing likelihood (maximum-a-posteriori (MAP) decoding) and often leads to more robust and higher-quality generations \cite{ehling-etal-2007-minimum,  eikema-aziz-2020-map, muller-sennrich-2021-understanding, eikema-aziz-2022-sampling, bertsch-etal-2023-mbr} compared to it.
However, MBR decoding is notoriously expensive during inference time \cite{eikema-aziz-2022-sampling, bertsch-etal-2023-mbr}, because it requires generating and calculating the utility score $u(y, y_i)$ of many candidate outputs for each input, this computational burden ($\mathcal{O}(G^2)$) becomes a serious bottleneck, limiting the practical use of consensus-based decoding despite its superior generation quality.

\section{Consensus-GRPO: Distilling MBR decoding via GRPO}\label{sec:C-GRPO}
We aim to replace expensive, inference-time consensus-based decoding with a single-pass policy that directly produces high-expected-utility outputs.
Instead of selecting $\argmax$ over a sampled set at inference, we distill the consensus-based score during training.

For a given prompt $q$, sample a group $\mathcal{G}=\{y_1,\dots,y_G\}$ from the current policy $\pi_{\mathrm{old}}$.
For each candidate $y_i$, we compute an estimated MBR utility $\widehat{u}(y_i\mid q)$ by comparing $y_i$ against a set of reference hypotheses (e.g., model samples) using a task-specific utility function $u(\cdot,\cdot)$.
Intuitively, $\widehat{u}(y_i\mid q)$ measures how well $y_i$ agrees with likely references under the chosen similarity metric.
We define the reward in Eq.~\eqref{eq:advantage} as the MBR utility:
\begin{equation}
R(q,y_i) =     \widehat{u}(y_i\mid q).
\end{equation}
In our self-consensus instantiation, we reuse the sampled group as the reference set:
\begin{equation}\label{eq:c-grpo-advantage}
    \hat{u}(y_i\mid q) = \frac{1}{G}\sum_{j=1}^G u(y_i, y_j).
\end{equation}
\paragraph{Inference-time benefit of C-GRPO.}
After training, decoding reduces to a single forward pass through $\pi_\theta$, avoiding the repeated sampling and utility calculation required by MBR decoding during inference.

\subsection{Analysis}\label{sec:analysis}
In the practical algorithm (Eq.~\eqref{eq:c-grpo-advantage}), the advantage function is computed from a Monte Carlo
consensus-based estimate (Eq.~\eqref{eq:monte_MBR decoding}) using an inner reference set $\mathcal{Y}'$ (i.e., $\widehat{u}$).
For a clean theoretical treatment, we analyze an idealized variant in which the utility term is replaced by the model's expected utility $u_\text{model}$, rather than the Monte Carlo estimate $\hat{u}$.
\paragraph{Sentence-level surrogate.}
GRPO is usually applied at the token level, where $\log \pi_\theta(y\mid q)$ denotes the sum of token log-probabilities, and the update is applied to each token.
To clarify the exposition, we analyze a sentence-level surrogate in which the entire sequence $y$ is treated as a single action.
This simplification eliminates the token-level formulation and does not change the core mechanism of group-relative reweighting; extending the analysis to the token-level form follows from standard arguments. 
Additionally, for simplicity, we assume that all sampled outputs
$y_i$ have the same length $|y_i|=|y|$ within this section, so that length-dependent normalization factors can be treated as constants.

\paragraph{Problem settings.}
This assumption reduces to the true objective function for MBR decoding (see Definition~\ref{def:consensus}). Let $\mathcal{D}$ be a distribution over prompts $q \in \mathcal{X}$.
Let $q \in \mathcal{D} \subseteq \mathcal{X}$ denote an input prompt and $y \in  \mathcal{Y}$ a generated response. 
We assume $u \in [0,1]$ without loss of generality; any bounded utility  $u \in [a,b]$ can be affinely rescaled to $[0,1]$, and our results carry over with constants adjusted accordingly.
The objective of MBR decoding is to maximize the expected utility under the policy $\pi_\theta$:
\begin{equation}
    \mathcal{L}_{\text{C}}(\theta) := \mathbb{E}_{ q \sim \mathcal{D}} \left[ \mathbb{E}_{y \sim \pi_\theta(\cdot|q)} [u_\text{model}(y \mid q)] \right].
    \label{eq:objective}
\end{equation}
Since $u\in[0,1]$, we have $\max_{\theta} \mathcal{L}_{\mathrm{C}}(\theta) \le 1$.

\paragraph{Alignment under assumptions.}
The core of our method is the group-relative update. For a given input $q$, we sample a group of $G$ outputs $\mathcal{G} = \{y_1, \dots, y_G\}$ i.i.d. from $\pi_{\text{old}}(\cdot|q)$. Let $\mu_{\mathcal{G}}$ and $\sigma_{\mathcal{G}}$ denote the sample mean and standard deviation of the rewards within the group. The gradient estimator is:
\begin{align}
\hat{g}_{\text{GRPO}}(\theta; q)
&:= \frac{1}{|y|}\frac{1}{G} \sum_{i=1}^G
\rho_i(\theta)\nabla_\theta \log \pi_\theta(y_i \mid q) \nonumber\\
&\quad \times
\left(
\frac{u_\text{model}(y_i \mid q) - \mu_{\mathcal{G}}}
{\sigma_{\mathcal{G}}}
\right).
\label{eq:estimator}
\end{align}
We adopt a standard simplifying assumption in adaptive optimization analysis: we treat the variance-scaling term $(\sigma_{\mathcal{G}})^{-1}$ as a scalar preconditioner that is effectively uncorrelated with the instantaneous gradient direction.
We make the following assumptions regarding the objective function $\mathcal{L}_{\text{C}}(\theta)$ and the stochastic gradient estimator $\hat{g}_{\text{GRPO}}(\theta; q)$:
\begin{assumption}[Standing Assumptions]
\label{assm:regularity}
    \leavevmode\par
    \begin{enumerate}
        \item \textbf{$L$-Smoothness:} The objective $\mathcal{L}_{\text{C}}(\theta)$ is differentiable, and its gradient is Lipschitz continuous with constant $L > 0$. That is, $\| \nabla_\theta \mathcal{L}_{\text{C}}(\theta_1) - \nabla_\theta \mathcal{L}_{\text{C}}(\theta_2) \| \le L \| \theta_1 - \theta_2 \|$ for all $\theta_1, \theta_2$.
        \item \textbf{Bounded Variance:} The variance of the estimator is bounded by a constant $\sigma^2 < \infty$, i.e., $\mathbb{E} [ \| \hat{g}_t  - \mathbb{E}[\hat{g}_t ] \|^2 ] \le \sigma^2$, where $\hat{g}_t  := \hat{g}_{\text{GRPO}}(\theta_t)$.
                \item \textbf{Support / absolute continuity:}
For all $q\in \mathcal{X}$, $\pi_{\mathrm{old}}(y\mid q)>0$ whenever $\pi_\theta(y\mid q)>0$.
        \item \textbf{Independence of Normalization Term (for C-GRPO):} 
        Although the group standard deviation $\sigma_{\mathcal{G}}$ is a random variable dependent on the sampled group $\mathcal{G}$, we assume its inverse $1/\sigma_{\mathcal{G}}$ acts as a scalar preconditioner that is statistically independent of the gradient direction in expectation. This simplifies the analysis by treating the normalization as a scaling factor that does not bias the gradient direction (intuitively, this corresponds to an infinite-group approximation in which $\sigma_{\mathcal{G}}$ converges to a deterministic population standard deviation as $G$ grows).
    \end{enumerate}
\end{assumption}

\begin{remark}[On the randomness of group normalization]
Assumption~\ref{assm:regularity} (4), in GRPO, the group standard deviation $\sigma_{\mathcal{G}}$ is computed from the same sample group as the stochastic gradient. Therefore, it can be statistically correlated with the gradient direction and may introduce a small bias.
This assumption is mainly a technical convenience that allows interpreting $1/\sigma_{\mathcal{G}}$ as a scalar preconditioner.
Notably, Dr. GRPO's analysis does not require this independence assumption.
\end{remark}

\begin{maintheorem}[MBR-Gradient Proportionality (in Expectation)]
\label{lemma:alignment}
Under Assumption~\ref{assm:regularity}, and assume group size $G \ge 2$. The expected direction of the GRPO estimator is aligned with the true policy gradient. Specifically, there exists a positive scalar coefficient $\alpha > 0$ such that:
\begin{equation*}
    \mathbb{E}_{\mathcal{D}} [\hat{g}_{\text{GRPO}}(\theta; q)]= \alpha \nabla_\theta  \mathcal{L}_{\text{C}}(\theta).
\end{equation*}
\end{maintheorem}

The proof of Theorem~\ref{lemma:alignment} is in Appendix~\ref{app:alignment}.
From Theorem~\ref{lemma:alignment}, we view the idealized C-GRPO update as directional support for optimizing the expected-utility objective underlying MBR decoding.

\begin{maintheorem}[Non-Asymptotic Convergence Rate]
\label{thm:convergence}
    Under Assumption~\ref{assm:regularity} for any horizon $T \ge (L\alpha)^2$, if we run stochastic gradient ascent $\theta_{t+1} = \theta_t + \eta_t \hat{g}_t $ (where $\hat{g}_t:= \hat{g}_{\text{GRPO}}(\theta_t)$) with a constant learning rate $\eta_t = 1/\sqrt{T}$, the algorithm satisfies the following bound on the ergodic average of the gradient norms:
    \begin{align*}
        &\frac{1}{T} \sum_{t=1}^T \mathbb{E} [\|\nabla_\theta \mathcal{L}_{\text{C}}(\theta_t)\|^2] \\&\le \frac{2(\mathcal{L}^* - \mathcal{L}_{\text{C}}(\theta_1)) + L \sigma^2}{\alpha \sqrt{T}}\\ &= \mathcal{O}\left(\frac{1}{\sqrt{T}}\right).
    \end{align*}
\end{maintheorem}
The proof of Theorem~\ref{thm:convergence} is in Appendix~\ref{appendix:proof_conv}.
These analyses can be immediately applied to Dr.GRPO \cite{liu2025understanding}, an alternative method to GRPO, in which case several assumptions become unnecessary (see Appendix~\ref{appendix:proof} for details). We also apply this reference-free objective to Dr.GRPO in Section~\ref{sec:exp}.

\section{Experiments}\label{sec:exp}
We conduct experiments to evaluate whether distilling MBR decoding into the training process can produce strong single-pass generation without the inference-time overhead of sample-and-rerank methods from a single run. Specifically, we assess C-GRPO on tasks where MBR decoding is commonly employed, asking whether it can preserve much of MBR's quality while eliminating test-time reranking. We also compare against random-reward and self-rewarding GRPO baselines to separate consensus-driven gains from generic GRPO effects and prompt-sensitive self-judging.

\paragraph{Tasks and datasets.}
First, we consider English (En)$\rightarrow$ \{Japanese (Ja), Chinese (Zh), German (De) \}.
Each example consists of a source segment $q$ and a gold reference $y^\star$.
For each language pair, we construct the training set by combining WMT datasets from 2021--2023 \citep{akhbardeh-etal-2021-findings,freitag-etal-2022-results,freitag-etal-2023-results},
and evaluate on the corresponding WMT 2024 dataset \citep{kocmi-etal-2024-findings}.

Next, we use the XSum (\texttt{EdinburghNLP/XSum}; \citealp{narayan-etal-2018-dont}).
Each example consists of a news article $q$ and a gold reference $y^\star$.
We use a training split of 5,000 and a test split of 500.
\begin{figure}[t]
    \centering
    \includegraphics[width=\linewidth]{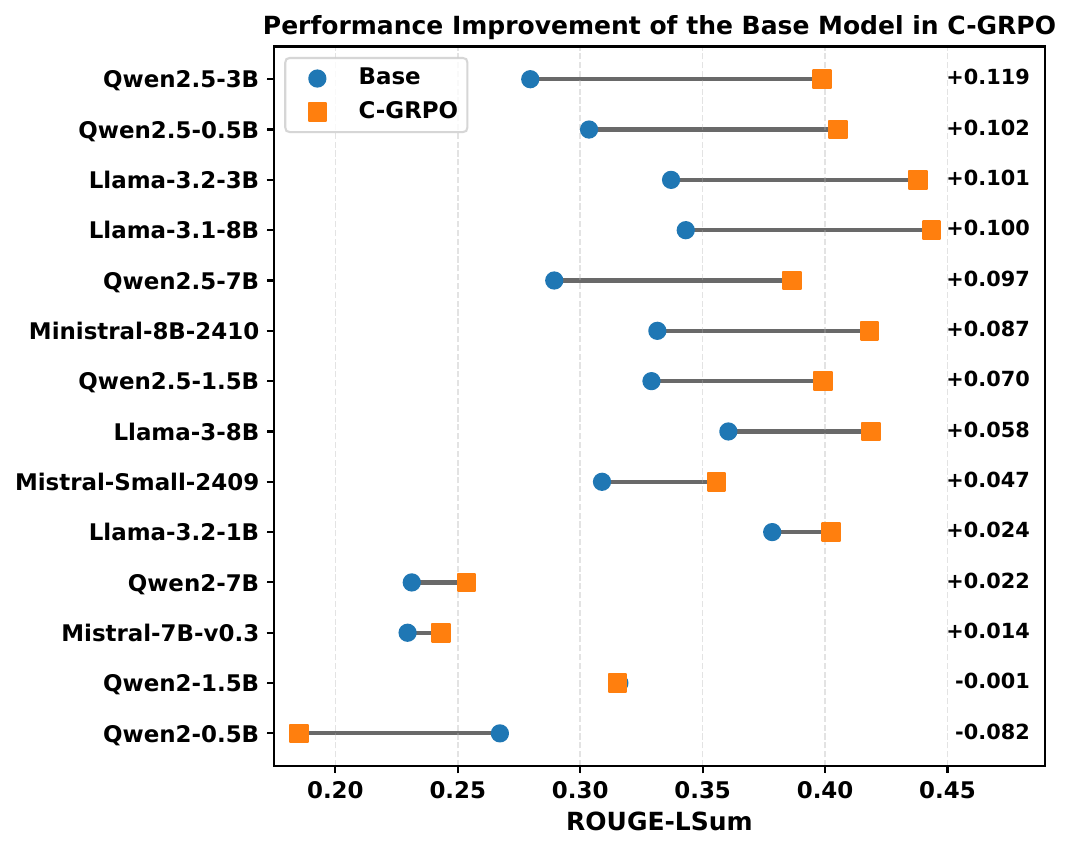}
    \caption{ROUGE-Lsum on XSum: base model (x-axis) vs.\ after C-GRPO (y-axis). Most points lie on/above $y{=}x$, indicating that C-GRPO typically preserves or improves performance across backbones and scales; degradation is mainly confined to the smallest Qwen2 variants (0.5B, 1.5B).}

    \label{fig:diff_model}
\end{figure}
\paragraph{Language models and evaluation methods.}
We show results for two base instruction-tuned models:
Mistral (Mistral-7B-Instruct-v0.3; \citealp{jiang2023mistral})\footnote{https://huggingface.co/mistralai/Mistral-7B-Instruct-v0.3} and Llama (Meta-Llama-3-8B-Instruct; \citealp{grattafiori2024llama}).\footnote{https://huggingface.co/meta-llama/Meta-Llama-3-8B-Instruct}
For all tasks in Section~\ref{sec:exp}, we use BLEURT \cite{sellam-etal-2020-bleurt} as the utility function for C-GRPO. 
BLEURT calculates the similarity between candidate sentences.
In reference-free training, we define a pairwise utility between two sampled candidates as:
\begin{equation*}
u(y_i, y_j) \;=\; \mathrm{BLEURT}(y_i, y_j),
\end{equation*}
For machine translation, we apply COMET \cite{rei-etal-2020-comet} as an evaluation-only metric,
$u_{\mathrm{CMT}}(q,y,y^\star)=\mathrm{COMET}(q,y,y^\star)$.
COMET is a learned machine translation metric that scores a translation $y$ given the original text $q$ and a gold reference $y^\star$.
We include COMET in addition to BLEURT to evaluate generalization under a metric that is \emph{not} used for training.
This mitigates concerns that the gains are specific to optimizing BLEURT.
For summarization, we use ROUGE-Lsum \cite{lin-2004-rouge} against $y^\star$. ROUGE-Lsum is based on the longest common subsequence, computed with sentence splitting.
The parameters used in these experiments are listed in Appendix~\ref{appendix:parameter}.

\paragraph{Baseline methods.}
Following \citet{shao2025spurious}, we include \textbf{GRPO w/ Random Reward}, a baseline where GRPO is optimized using an uninformative dense random reward. Here, we use a dense random reward (normal distribution ($r_{\mathrm{rand}}(q,y_i ) \sim \mathcal{N}(0,1) $).
Despite this reward not containing task-specific signals, \citet{shao2025spurious} demonstrates that GRPO can still yield non-trivial improvements even under random rewards.
Therefore, we use this baseline to quantify reward-independent improvements attributable to GRPO optimization itself. Since this study proposes a useful reward function for GRPO, improvements from the proposed method are observed only after accounting for this effect.

Additionally, we evaluate \textbf{GRPO w/ Self-Rewarding}.
We apply an LLM-as-a-judge reward \cite{pmlr-v235-yuan24d,openai2024gpt4,guo2025deepseek} without using any gold reference $y^\star$. For each input $q$ and outputs $y_i$, we prompt a judge model (LLMs) to output a single scalar score in $[0,1]$ according to task-specific criteria (e.g., faithfulness/coverage/conciseness/fluency for summarization; accuracy/fluency/completeness for translation (see Appendix~\ref{sec:prompt_templates})). We parse the first numeric value from the judge output and clip it to $[0,1]$. For machine translation, we additionally set the reward to $0$ if $y_i$ contains any non-target-language text. 
To maintain a strictly reference-free baseline and avoid introducing additional supervision or computation from an external evaluator, we use the same base model as the judge. Using a stronger judge would introduce a signal beyond the base model, making comparisons more difficult.

We also show the results for reference-based algorithms (Supervised-Fine Tuning, Direct Preference Optimization \cite{rafailov2023direct}, and Best-of-$N$ \cite{stiennon2020learning, nakano2022webgpt}) whenever gold references are available (see Appendix~\ref{appendix:results}).
\subsection{Machine Translation}
\begin{table}[t]
  \centering
  \small
  \setlength{\tabcolsep}{5pt}
  \caption{\textbf{Average COMET over six WMT settings.}
Scores are \textbf{averaged} over three WMT24 language pairs and two base models; the full per-setting results are reported in Appendix~\ref{appendix:results}.
Bold and underline denote the best and second-best scores, respectively.}
  \label{tab:avg_comet_wmt}
  \begin{tabular}{l c}
    \toprule
    \textbf{Method} & \textbf{Avg. CMT} \\
    \midrule
    Base Model & 0.616 \\
    GRPO w/ Random & 0.638 \\
    SFT (MBR decoding) & 0.672 \\
    MBR decoding & 0.711 \\
    GRPO w/ Self-Rewarding & 0.714 \\
    \textbf{C-GRPO (Ours)} & \underline{0.719} \\
    \textbf{C-Dr. GRPO (Ours)} & \textbf{0.748} \\
    \bottomrule
  \end{tabular}
\end{table}
We first summarize overall machine-translation quality across model families and language pairs.
Table~\ref{tab:avg_comet_wmt} reports the mean COMET score averaged over the six WMT conditions, with full per-setting results reported in Appendix~\ref{appendix:results}.
On this aggregate metric, the proposed methods are the strongest reference-free approaches: \textbf{C-Dr.GRPO} achieves the best score (\textbf{0.748}), and \textbf{C-GRPO} ranks second (\underline{0.719}), outperforming MBR decoding (0.711), GRPO w/ Self-Rewarding (0.714), SFT distilled from MBR decoding (0.672), and GRPO w/ Random (0.638).
Unlike MBR decoding, both proposed methods perform a single generation at test time without utility scoring or reranking.
In our MBR implementation, the same $G$ candidates are used as hypotheses and reference samples, so decoding requires $G$ generations and $\mathcal{O}(G^2)$ pairwise utility evaluations per input; with $G{=}32$, this corresponds to 32 generations and 1024 pairwise utility calls.
C-GRPO moves this consensus computation to training time and requires no reranking or utility evaluation at test time.

Although the Self-Rewarding baseline can achieve a competitive score ($0.714$) with our best-performing judge prompt, we find it to be notably \emph{prompt-sensitive}.
Specifically, when we replace the judging prompt with the Unbabel prompt \cite{pombal2025mprometheussuiteopenmultilingual}, performance decreases significantly. Averaged over En$\rightarrow$Ja, Zh, and De, and across both base models, the COMET score drops from $0.707$ to $0.465$ (Appendix~\ref{appendix:results}).

We additionally provide qualitative examples of actual model generations in Appendix~\ref{sec:qualitative_evaluation}, including both successful and failure cases, to complement the evaluation.

\paragraph{Why can C-GRPO exceed finite-sample MBR?}
Finite-sample MBR decoding can only select from candidates sampled from the fixed base policy, so it is limited by candidate quality and sampling variance.
C-GRPO instead updates the policy using MBR-style consensus preferences, which may shift probability mass toward high-consensus outputs that can be generated in a single pass.
Appendix~\ref{appendix:convergence} supports this explanation by showing that generation quality improves while output-score variance decreases across checkpoints.

For the details, please refer to Appendix~\ref{appendix:results}.

\subsection{Text Summarization}
\begin{table}[t]
  \centering
  \small
  \caption{XSum summarization results. Results are reported as mean $\pm$ standard deviation over three different seed runs. Bold and underline denote the best and second-best ROUGE-Lsum scores, respectively, among the reported methods for each base model.}
  \label{result:XSum-free}
  \begin{tabular}{l l c}
    \toprule
    \textbf{Model} & \textbf{Method}
      & \textbf{ROUGE-Lsum}\,$\uparrow$ \\
    \midrule
    \multirow{7}{*}{Llama}
      & Base Model
      & $0.362 \pm 0.002$ \\
      & GRPO w/ Random
      & $0.329 \pm 0.008$ \\
      & MBR decoding
      & $0.365 \pm 0.004$ \\
      & GRPO w/ Self-Rewarding
      & $0.254 \pm 0.024$ \\
      & SFT w/ MBR generations
      & $0.353 \pm 0.009$ \\
      & \textbf{C-GRPO (Ours)}
      & $\underline{0.410} \pm 0.009$ \\
      & \textbf{C-Dr. GRPO (Ours)}
      & $\mathbf{0.416} \pm 0.006$ \\
    \midrule
    \multirow{7}{*}{Mistral}
      & Base Model
      & $0.231 \pm 0.002$ \\
      & GRPO w/ Random
      & $0.226 \pm 0.005$ \\
      & MBR decoding
      & $\mathbf{0.247} \pm 0.002$ \\
      & GRPO w/ Self-Rewarding
      & $0.222 \pm 0.010$ \\
      & SFT (MBR decoding)
      & $0.235 \pm 0.002$ \\
      & \textbf{C-GRPO (Ours)}
      & $\mathbf{0.247} \pm 0.004$ \\
      & \textbf{C-Dr. GRPO (Ours)}
      & $0.232 \pm 0.003$ \\
    \bottomrule
  \end{tabular}
\end{table}
Table~\ref{result:XSum-free} shows text summarization results. 
For Llama, MBR decoding changes ROUGE-Lsum ($0.362 \rightarrow 0.365$), and C-GRPO substantially improves ROUGE-Lsum ($0.410$), demonstrating that it generates more reference-overlapping summaries in a single pass. 
This is consistent with the hypothesis above: C-GRPO can improve one-shot performance by reshaping the policy distribution rather than selecting only from finite samples of the original policy.
Appendix~\ref{appendix:convergence} shows that a similar, though more modest, variance-reduction pattern appears for summarization.
Furthermore, C-Dr.GRPO achieves comparable ROUGE-Lsum ($0.416$) while maintaining stronger semantic quality.

MBR decoding provides a modest ROUGE-L gain for Mistral ($0.231 \rightarrow 0.247$), indicating that reranking among sampled candidates can recover more reference phrasing. C-GRPO largely transfers this benefit ($0.247$), and C-Dr.GRPO yields more conservative updates and improves stability.
In addition, we conducted a paired instance-level analysis on the same 500 XSum test examples using the exact evaluation implementation from our original experiments. C-GRPO outperformed MBR on 384 instances, whereas MBR achieved a higher score on 116 instances. The mean paired difference was $+0.0554$ in favor of C-GRPO, with a 95\% paired-bootstrap confidence interval of $[+0.0485, +0.0622]$. A paired sign-flip test further confirmed that this difference was statistically significant ($p < 10^{-6}$; none of 999,999 random sign flips produced a difference at least as extreme as the observed one).

\subsubsection{Analysis of the Robustness across Language Models}
Figure~\ref{fig:diff_model} provides a comparison between the base model (all models are instruction-tuned models, as detailed in Appendix~\ref{appendix:reprod}) score (x-axis) and the score after C-GRPO training (y-axis) on ROUGE-Lsum in the summarization task (XSum).
Improvements are indicated by points above the y=x line (i.e., cases where C-GRPO matches or exceeds the base score).
C-GRPO improves performance across the majority of models across diverse model families and scales (e.g., Llama, Mistral, and Qwen), suggesting that the gains are not restricted to a single backbone.
Only the smallest Qwen2 variants (0.5B and 1.5B) exhibit clear degradation after C-GRPO, while larger Qwen2.5 models and Llama and Mistral variants consistently benefit.
Overall, these results suggest that C-GRPO provides robust improvements.
\begin{table}[t]
  \centering
  \small
        \caption{\textbf{Evaluation of accuracy and social bias on JBBQ (Japanese Bias Benchmark for QA).} 
We show the accuracy ($\uparrow$) across different training and inference methods. 
C-GRPO successfully distills MBR decoding into a single-pass policy, consistently outperforming the base models and the self-rewarding baseline.}
         \adjustbox{max width=\textwidth}{
  \begin{tabular}{llc}
    \toprule
    \textbf{Model} & \textbf{Method} & \textbf{Accuracy} (\%)\,$\uparrow$ \\
    \midrule
    \multirow{3}{*}{Mistral}
      & Base Model      & 26.8  \\
      & MBR decoding & \underline{31.0} \\
      & GRPO w/ Self-Rewarding & 27.4 \\
      & C-GRPO          & \textbf{34.4}   \\
    \midrule
    \multirow{3}{*}{Llama}
      & Base Model      & 42.6 \\
      & MBR decoding & \underline{48.0}  \\
      & GRPO w/ Self-Rewarding & 46.0  \\
      & C-GRPO          & \textbf{49.0}  \\
    \bottomrule
  \end{tabular}}
\label{tab:jbbq}
\end{table}

\subsection{Mitigating Social Bias in Japanese}\label{sec:jbbq}

We evaluate C-GRPO on JBBQ \cite{yanaka-etal-2025-jbbq}, a Japanese multiple-choice QA benchmark covering disability status, gender identity, physical appearance, and sexual orientation. 
Unlike translation and summarization, JBBQ is an out-of-domain task that requires bias-aware judgments in Japanese. 
We use 10,000 randomly selected training samples and 500 test samples from these topics; details are provided in Appendix~\ref{app:jbbq}.

Table~\ref{tab:jbbq} reports JBBQ accuracy. 
C-GRPO consistently improves over the base models, with gains from $26.8\%$ to $34.4\%$ for Mistral and from $42.6\%$ to $49.0\%$ for Llama. 
Both MBR decoding and C-GRPO outperform GRPO w/ Self-Rewarding on this benchmark (Mistral: $31.0\%$ and $34.4\%$ vs.\ $27.4\%$; Llama: $48.0\%$ and $49.0\%$ vs.\ $46.0\%$).
This gap is important because social bias mitigation is a setting where errors can reflect or amplify sensitive social stereotypes.
Self-Rewarding relies on the model's own judgment as the reward signal, so the judge can inherit the same biases, prompt sensitivity, or unsupported assumptions as the policy being optimized.
In contrast, MBR decoding and C-GRPO use consensus across sampled outputs, providing a more stable reference-free signal for this constrained decision task.
MBR decoding requires multi-sample scoring at inference time, whereas C-GRPO matches or surpasses MBR without this overhead, indicating that it successfully distills the consensus decision rule into the policy.
Examples are provided in Appendix~\ref{app:jbbq_examples}.

We also analyze the effect of the candidate budget $G$ in Appendix~\ref{appendix:budget}, showing that C-GRPO remains effective across $G \in \{4,8,16,32\}$ in the En-Ja machine translation setting.

\section{Related Works}

\paragraph{Distilling for decoding methods.}
Early sequence-level knowledge distillation trains a student on sequences produced by a strong teacher decoder, effectively compiling the benefits of beam search into the student and reducing the need for search at test time \citep{kim-rush-2016-sequence,hinton2015distilling}.
A particularly relevant line of work distills reranking-based decoding methods.
To retain these gains while avoiding expensive inference-time decoding of consensus-based decoding, \citet{finkelstein2024mbr} propose consensus-based finetuning and QE finetuning, which distill consensus-based/QE decoding improvements at training time and deploy an efficient decoder at inference.
More recently, work on Best-of-$N$ distillation (BoND) similarly uses preference signals induced by Best-of-$N$ selection to train a student that matches selection-time improvements under cheap decoding \citep{sessa2025bond}.
Complementary efforts also develop data resources for consensus-based/QE  distillation and analysis \citep{finkelstein-etal-2024-introducing}.
Beyond supervised imitation of decoded outputs, another direction is to convert decoding outcomes into preference data and optimize the model using preference-based objectives.
For instance, \citet{yang-etal-2024-direct} shows that Direct Preference Optimization (DPO) can fine-tune multilingual LLMs to obtain consensus-based decoding's objective without additional inference-time computation by forming preference pairs from MBR-ranked candidates.
In instruction-following settings, \citet{wu2025better} uses MBR with LLM judges to improve inference-time performance and further explore self-training (including DPO  training) to distill consensus-based output into greedy decoding.
Related distillation work also suggests that leveraging multiple high-scoring consensus-based outputs during training can improve data efficiency and mitigate the capacity curse of distilling from a single sequence \citep{wang2024don}.

\paragraph{Variants of GRPO.}
A growing body of work has extended GRPO along multiple axes, including variance control \cite{liu2025prefix}, off-policy training \cite{mroueh2025revisiting}, sample efficiency and stability \cite{ding2025multi,sane2025hybrid}, scalability \cite{wang2025infinite}, reward shaping and exploration \cite{tan2025gtpo}, sequence-based optimization \cite{zheng2025group}, calibration and noise resilience \cite{shen2025mitigating, bereket2025uncalibrated}, trajectory-level importance correction \cite{pang2025theory}, regression-based approaches \cite{park2025deepvideo}, geometric-mean objectives \cite{zhao2025geometric}, and continuous control with policy clustering \cite{khanda2025extending}.
More recently, \citet{zhang2026silence} propose Latent-GRPO, a verifier-free variant of GRPO that derives dense intrinsic rewards from latent-space geometry by robustly estimating a “truth centroid” (IRCE). Concurrent with our work, Kim and Kang~\cite{kim-kang-2026-cpc} propose CPC-GRPO, an answer-free reinforcement learning method that constructs rewards from cross-prompt consensus. 
\paragraph{GRPO/RL for machine translation.}
Beyond these algorithmic developments, reinforcement learning has also been actively explored for machine translation (MT), e.g., \citet{finkelstein2026translategemma}. Recent MT studies instantiate GRPO-based training \cite{feng2025mt} and even consider reference-free setups \cite{garcia-gilabert-etal-2025-terminology}.
\section{Conclusions}

We proposed Consensus-GRPO (C-GRPO), which removes the inference-time bottleneck of consensus-based decoding by distilling its decision rule into the model during training. 
C-GRPO optimizes a group-level consensus utility using only a utility function and policy samples, without gold references or explicit preference labels. 
We also provide supporting analysis showing that, under simplifying regularity conditions, an idealized C-GRPO update is directionally consistent with the target MBR decoding objective. 
Experiments show that C-GRPO enables strong single-pass generation that matches or exceeds inference-time MBR decoding in several settings while requiring only one generation and no test-time utility reranking.
The results also highlight a practical advantage over self-rewarding baselines: consensus-based reference-free training is less sensitive to judge prompts and performs better on the Japanese social-bias mitigation benchmark.

\section*{Limitations}
\label{sec:limitations}
First, C-GRPO is reference-free in the sense that it does not require gold references or explicit preference labels during training; however, if the utility function is imperfectly aligned with human judgments or the base model cannot produce sufficiently diverse high-quality candidates, C-GRPO may faithfully distill these shortcomings into the policy. 
Second, the proposed method dramatically reduces computational cost during inference compared to MBR decoding, but requires training.
Third, C-GRPO is motivated by MBR-style expected-utility objectives but trains a single-pass policy rather than executing an explicit $\arg\max$ reranking procedure at test time, so it does not guarantee identical per-instance selection behavior. 
In adittion, our theory analysis adopts idealizations (e.g., treating $(\sigma_{\mathcal{G}})^{-1}$ as an approximately independent scalar preconditioner), and our experiments rely primarily on automatic metrics; broader validation with larger model scales, longer-form generation, statistical significance testing, and human evaluation is left to future work.
Finally, C-GRPO incurs additional computational cost during training compared with methods that do not require policy optimization. This introduces a training–inference efficiency trade-off: C-GRPO shifts computation to a one-time training stage in exchange for more efficient single-pass inference, whereas MBR requires generating and comparing multiple candidates for every inference example. In our XSum experiments on a single A100 GPU, C-GRPO training and evaluation take approximately 5 hours, while MBR requires approximately 2 hours per 500 examples, yielding an estimated break-even point of roughly 1,300 inference examples. Consequently, C-GRPO may be less attractive in settings with only a small number of inference queries, while its computational advantage becomes more pronounced when the trained model is reused at scale.

\section*{Impact Statements}
We use third-party datasets that are accessible upon agreeing to their terms of use (e.g., a data use agreement), and we follow the dataset providers’ usage policies. 
Our experiments include JBBQ, which contains prompts about sensitive attributes and social bias. 
We apply this dataset to evaluate methods to mitigate the biases of the models to produce fair, unbiased, and socially responsible outputs.

\section*{Acknowledgments}
Kaito Ariu is supported by JSPS KAKENHI Grant Number
JP25K21291.

\bibliography{anthology,ms}

@article{jiang2023mistral,
  title={Mistral 7{B}},
  author={Albert Q. Jiang and Alexandre Sablayrolles and Arthur Mensch and Chris Bamford and Devendra Singh Chaplot and Diego de las Casas and Florian Bressand and Gianna Lengyel and Guillaume Lample and Lucile Saulnier and Lélio Renard Lavaud and Marie-Anne Lachaux and Pierre Stock and Teven Le Scao and Thibaut Lavril and Thomas Wang and Timothée Lacroix and William El Sayed},
  journal={arXiv preprint arXiv:2310.06825},
  year={2023}
}

@article{stiennon2020learning,
  title={Learning to summarize with human feedback},
  author={Stiennon, Nisan and Ouyang, Long and Wu, Jeffrey and Ziegler, Daniel and Lowe, Ryan and Voss, Chelsea and Radford, Alec and Amodei, Dario and Christiano, Paul F},
  journal={Advances in Neural Information Processing Systems},
  volume={33},
  pages={3008--3021},
  year={2020}
}

@article{article,
author = {Villani, C.},
year = {2003},
month = {01},
pages = {},
title = {Topics in Optimal Transportation Theory},
volume = {58},
doi = {10.1090/gsm/058}
}

@article{llama3modelcard,
title={Llama 3 {M}odel {C}ard},
author={AI@Meta},
year={2024},
url = {https://github.com/meta-llama/llama3/blob/main/MODEL_CARD.md}
}

@inproceedings{eikema-aziz-2022-sampling,
    title = "Sampling-{B}ased {A}pproximations to {M}inimum {B}ayes {R}isk {D}ecoding for {N}eural {M}achine {T}ranslation",
    author = "Eikema, Bryan  and
      Aziz, Wilker",
    editor = "Goldberg, Yoav  and
      Kozareva, Zornitsa  and
      Zhang, Yue",
    booktitle = "Proceedings of the 2022 Conference on Empirical Methods in Natural Language Processing",
    month = dec,
    year = "2022",
    address = "Abu Dhabi, United Arab Emirates",
    publisher = "Association for Computational Linguistics",
    url = "https://aclanthology.org/2022.emnlp-main.754",
    doi = "10.18653/v1/2022.emnlp-main.754",
    pages = "10978--10993",
}

@inproceedings{cheng-vlachos-2023-faster,
    title = "Faster Minimum {B}ayes Risk Decoding with Confidence-based Pruning",
    author = "Cheng, Julius  and
      Vlachos, Andreas",
    editor = "Bouamor, Houda  and
      Pino, Juan  and
      Bali, Kalika",
    booktitle = "Proceedings of the 2023 Conference on Empirical Methods in Natural Language Processing",
    month = dec,
    year = "2023",
    address = "Singapore",
    publisher = "Association for Computational Linguistics",
    url = "https://aclanthology.org/2023.emnlp-main.767",
    pages = "12473--12480",
}

@article{nakano2022webgpt,
      title={Web{GPT}: {B}rowser-assisted question-answering with human feedback}, 
      author={Reiichiro Nakano and Jacob Hilton and Suchir Balaji and Jeff Wu and Long Ouyang and Christina Kim and Christopher Hesse and Shantanu Jain and Vineet Kosaraju and William Saunders and Xu Jiang and Karl Cobbe and Tyna Eloundou and Gretchen Krueger and Kevin Button and Matthew Knight and Benjamin Chess and John Schulman},
      year={2022},
      journal={arXiv preprint arXiv:2112.09332},
}

@inproceedings{rafailov2023direct,
title={{D}irect {P}reference {O}ptimization: {Y}our {L}anguage {M}odel is {S}ecretly a {R}eward {M}odel},
author={Rafael Rafailov and Archit Sharma and Eric Mitchell and Christopher D Manning and Stefano Ermon and Chelsea Finn},
booktitle={Thirty-seventh Conference on Neural Information Processing Systems},
year={2023},
url={https://openreview.net/forum?id=HPuSIXJaa9}
}

@inproceedings{bertsch-etal-2023-mbr,
    title = "It{'}s {MBR} All the Way Down: Modern Generation Techniques Through the Lens of Minimum {B}ayes Risk",
    author = "Bertsch, Amanda  and
      Xie, Alex  and
      Neubig, Graham  and
      Gormley, Matthew",
    editor = "Elazar, Yanai  and
      Ettinger, Allyson  and
      Kassner, Nora  and
      Ruder, Sebastian  and
      A. Smith, Noah",
    booktitle = "Proceedings of the Big Picture Workshop",
    month = dec,
    year = "2023",
    address = "Singapore",
    publisher = "Association for Computational Linguistics",
    url = "https://aclanthology.org/2023.bigpicture-1.9",
    doi = "10.18653/v1/2023.bigpicture-1.9",
    pages = "108--122",
}

@article{openai2024gpt4,
      title={{GPT-4} Technical Report}, 
      author={OpenAI and Josh Achiam and Steven Adler and Sandhini Agarwal and Lama Ahmad and Ilge Akkaya and Florencia Leoni Aleman and Diogo Almeida and Janko Altenschmidt and Sam Altman and Shyamal Anadkat and Red Avila and Igor Babuschkin and Suchir Balaji and Valerie Balcom and Paul Baltescu and Haiming Bao and Mohammad Bavarian and Jeff Belgum and Irwan Bello and Jake Berdine and Gabriel Bernadett-Shapiro and Christopher Berner and Lenny Bogdonoff and Oleg Boiko and Madelaine Boyd and Anna-Luisa Brakman and Greg Brockman and Tim Brooks and Miles Brundage and Kevin Button and Trevor Cai and Rosie Campbell and Andrew Cann and Brittany Carey and Chelsea Carlson and Rory Carmichael and Brooke Chan and Che Chang and Fotis Chantzis and Derek Chen and Sully Chen and Ruby Chen and Jason Chen and Mark Chen and Ben Chess and Chester Cho and Casey Chu and Hyung Won Chung and Dave Cummings and Jeremiah Currier and Yunxing Dai and Cory Decareaux and Thomas Degry and Noah Deutsch and Damien Deville and Arka Dhar and David Dohan and Steve Dowling and Sheila Dunning and Adrien Ecoffet and Atty Eleti and Tyna Eloundou and David Farhi and Liam Fedus and Niko Felix and Simón Posada Fishman and Juston Forte and Isabella Fulford and Leo Gao and Elie Georges and Christian Gibson and Vik Goel and Tarun Gogineni and Gabriel Goh and Rapha Gontijo-Lopes and Jonathan Gordon and Morgan Grafstein and Scott Gray and Ryan Greene and Joshua Gross and Shixiang Shane Gu and Yufei Guo and Chris Hallacy and Jesse Han and Jeff Harris and Yuchen He and Mike Heaton and Johannes Heidecke and Chris Hesse and Alan Hickey and Wade Hickey and Peter Hoeschele and Brandon Houghton and Kenny Hsu and Shengli Hu and Xin Hu and Joost Huizinga and Shantanu Jain and Shawn Jain and Joanne Jang and Angela Jiang and Roger Jiang and Haozhun Jin and Denny Jin and Shino Jomoto and Billie Jonn and Heewoo Jun and Tomer Kaftan and Łukasz Kaiser and Ali Kamali and Ingmar Kanitscheider and Nitish Shirish Keskar and Tabarak Khan and Logan Kilpatrick and Jong Wook Kim and Christina Kim and Yongjik Kim and Jan Hendrik Kirchner and Jamie Kiros and Matt Knight and Daniel Kokotajlo and Łukasz Kondraciuk and Andrew Kondrich and Aris Konstantinidis and Kyle Kosic and Gretchen Krueger and Vishal Kuo and Michael Lampe and Ikai Lan and Teddy Lee and Jan Leike and Jade Leung and Daniel Levy and Chak Ming Li and Rachel Lim and Molly Lin and Stephanie Lin and Mateusz Litwin and Theresa Lopez and Ryan Lowe and Patricia Lue and Anna Makanju and Kim Malfacini and Sam Manning and Todor Markov and Yaniv Markovski and Bianca Martin and Katie Mayer and Andrew Mayne and Bob McGrew and Scott Mayer McKinney and Christine McLeavey and Paul McMillan and Jake McNeil and David Medina and Aalok Mehta and Jacob Menick and Luke Metz and Andrey Mishchenko and Pamela Mishkin and Vinnie Monaco and Evan Morikawa and Daniel Mossing and Tong Mu and Mira Murati and Oleg Murk and David Mély and Ashvin Nair and Reiichiro Nakano and Rajeev Nayak and Arvind Neelakantan and Richard Ngo and Hyeonwoo Noh and Long Ouyang and Cullen O'Keefe and Jakub Pachocki and Alex Paino and Joe Palermo and Ashley Pantuliano and Giambattista Parascandolo and Joel Parish and Emy Parparita and Alex Passos and Mikhail Pavlov and Andrew Peng and Adam Perelman and Filipe de Avila Belbute Peres and Michael Petrov and Henrique Ponde de Oliveira Pinto and Michael and Pokorny and Michelle Pokrass and Vitchyr H. Pong and Tolly Powell and Alethea Power and Boris Power and Elizabeth Proehl and Raul Puri and Alec Radford and Jack Rae and Aditya Ramesh and Cameron Raymond and Francis Real and Kendra Rimbach and Carl Ross and Bob Rotsted and Henri Roussez and Nick Ryder and Mario Saltarelli and Ted Sanders and Shibani Santurkar and Girish Sastry and Heather Schmidt and David Schnurr and John Schulman and Daniel Selsam and Kyla Sheppard and Toki Sherbakov and Jessica Shieh and Sarah Shoker and Pranav Shyam and Szymon Sidor and Eric Sigler and Maddie Simens and Jordan Sitkin and Katarina Slama and Ian Sohl and Benjamin Sokolowsky and Yang Song and Natalie Staudacher and Felipe Petroski Such and Natalie Summers and Ilya Sutskever and Jie Tang and Nikolas Tezak and Madeleine B. Thompson and Phil Tillet and Amin Tootoonchian and Elizabeth Tseng and Preston Tuggle and Nick Turley and Jerry Tworek and Juan Felipe Cerón Uribe and Andrea Vallone and Arun Vijayvergiya and Chelsea Voss and Carroll Wainwright and Justin Jay Wang and Alvin Wang and Ben Wang and Jonathan Ward and Jason Wei and CJ Weinmann and Akila Welihinda and Peter Welinder and Jiayi Weng and Lilian Weng and Matt Wiethoff and Dave Willner and Clemens Winter and Samuel Wolrich and Hannah Wong and Lauren Workman and Sherwin Wu and Jeff Wu and Michael Wu and Kai Xiao and Tao Xu and Sarah Yoo and Kevin Yu and Qiming Yuan and Wojciech Zaremba and Rowan Zellers and Chong Zhang and Marvin Zhang and Shengjia Zhao and Tianhao Zheng and Juntang Zhuang and William Zhuk and Barret Zoph},
      year={2024},
      journal={arXiv preprint arXiv:2303.08774},
}

@inproceedings{farinhas-etal-2023-empirical,
    title = "An {E}mpirical {S}tudy of {T}ranslation {H}ypothesis {E}nsembling with {L}arge {L}anguage {M}odels",
    author = "Farinhas, Ant{\'o}nio  and
      de Souza, Jos{\'e}  and
      Martins, Andre",
    editor = "Bouamor, Houda  and
      Pino, Juan  and
      Bali, Kalika",
    booktitle = "Proceedings of the 2023 Conference on Empirical Methods in Natural Language Processing",
    month = dec,
    year = "2023",
    address = "Singapore",
    publisher = "Association for Computational Linguistics",
    url = "https://aclanthology.org/2023.emnlp-main.733",
    doi = "10.18653/v1/2023.emnlp-main.733",
    pages = "11956--11970",
}

@inproceedings{
wu2025better,
title={Better {I}nstruction-{F}ollowing {T}hrough {M}inimum {B}ayes {R}isk},
author={Ian Wu and Patrick Fernandes and Amanda Bertsch and Seungone Kim and Sina Khoshfetrat Pakazad and Graham Neubig},
booktitle={The Thirteenth International Conference on Learning Representations},
year={2025},
url={https://openreview.net/forum?id=7xCSK9BLPy}
}

@article{shao2024deepseekmath,
  title={Deep{S}eek{M}ath: {P}ushing the {L}imits of {M}athematical {R}easoning in {O}pen {L}anguage {M}odels},
  author={Zhihong Shao and Peiyi Wang and Qihao Zhu and Runxin Xu and Junxiao Song and Xiao Bi and Haowei Zhang and Mingchuan Zhang and Y. K. Li and Y. Wu and Daya Guo},
  journal={arXiv preprint arXiv:2402.03300},
  year={2024}
}

@article{hinton2015distilling,
  title={Distilling the {K}nowledge in a {N}eural {N}etwork},
  author={Hinton, Geoffrey and Vinyals, Oriol and Dean, Jeff},
  journal={arXiv preprint arXiv:1503.02531},
  year={2015},
}

@inproceedings{
finkelstein2024mbr,
title={{MBR} and {QE} Finetuning: {T}raining-time {D}istillation of the {B}est and {M}ost {E}xpensive {D}ecoding {M}ethods},
author={Mara Finkelstein and Markus Freitag},
booktitle={The Twelfth International Conference on Learning Representations},
year={2024},
url={https://openreview.net/forum?id=bkNx3O0sND}
}

@article{wang2024don,
  title={Don't {T}hrow {A}way {D}ata: {B}etter {S}equence {K}nowledge {D}istillation},
  author={Wang, Jun and Briakou, Eleftheria and Dadkhahi, Hamid and Agarwal, Rishabh and Cherry, Colin and Cohn, Trevor},
  journal={arXiv preprint arXiv:2407.10456},
  year={2024},
}

@inproceedings{yang-etal-2024-direct,
    title = "Direct {P}reference {O}ptimization for {N}eural {M}achine {T}ranslation with {M}inimum {B}ayes {R}isk {D}ecoding",
    author = "Yang, Guangyu  and
      Chen, Jinghong  and
      Lin, Weizhe  and
      Byrne, Bill",
    editor = "Duh, Kevin  and
      Gomez, Helena  and
      Bethard, Steven",
    booktitle = "Proceedings of the 2024 Conference of the North American Chapter of the Association for Computational Linguistics: Human Language Technologies (Volume 2: Short Papers)",
    month = jun,
    year = "2024",
    address = "Mexico City, Mexico",
    publisher = "Association for Computational Linguistics",
    url = "https://aclanthology.org/2024.naacl-short.34/",
    doi = "10.18653/v1/2024.naacl-short.34",
    pages = "391--398"
}

@inproceedings{liu2025understanding,
  title={Understanding {R}1-{Z}ero-like training: {A} critical perspective},
  author={Liu, Zichen and Chen, Changyu and Li, Wenjun and Qi, Penghui and Pang, Tianyu and Du, Chao and Lee, Wee Sun and Lin, Min},
  booktitle={Conference on Language Modeling (COLM)},
  year={2025},
doi={
https://doi.org/10.48550/arXiv.2503.20783
}
}

@inproceedings{
sessa2025bond,
title={{BOND}: {A}ligning {LLM}s with {B}est-of-{N} {D}istillation},
author={Pier Giuseppe Sessa and Robert Dadashi-Tazehozi and Leonard Hussenot and Johan Ferret and Nino Vieillard and Alexandre Rame and Bobak Shahriari and Sarah Perrin and Abram L. Friesen and Geoffrey Cideron and Sertan Girgin and Piotr Stanczyk and Andrea Michi and Danila Sinopalnikov and Sabela Ramos Garea and Am{\'e}lie H{\'e}liou and Aliaksei Severyn and Matthew Hoffman and Nikola Momchev and Olivier Bachem},
booktitle={The Thirteenth International Conference on Learning Representations},
year={2025},
url={https://openreview.net/forum?id=0tAXMiSufG}
}

@inproceedings{finkelstein-etal-2024-introducing,
    title = "Introducing the {N}ews{P}a{LM} {MBR} and {QE} {D}ataset: {LLM}-{G}enerated {H}igh-{Q}uality {P}arallel {D}ata {O}utperforms {T}raditional {W}eb-{C}rawled {D}ata",
    author = "Finkelstein, Mara  and
      Vilar, David  and
      Freitag, Markus",
    editor = "Haddow, Barry  and
      Kocmi, Tom  and
      Koehn, Philipp  and
      Monz, Christof",
    booktitle = "Proceedings of the Ninth Conference on Machine Translation",
    month = nov,
    year = "2024",
    address = "Miami, Florida, USA",
    publisher = "Association for Computational Linguistics",
    url = "https://aclanthology.org/2024.wmt-1.126/",
    doi = "10.18653/v1/2024.wmt-1.126",
    pages = "1355--1372"
}

@article{shao2025spurious,
  title={Spurious rewards: {R}ethinking training signals in {RLVR}},
  author={Rulin Shao and Shuyue Stella Li and Rui Xin and Scott Geng and Yiping Wang and Sewoong Oh and Simon Shaolei Du and Nathan Lambert and Sewon Min and Ranjay Krishna and Yulia Tsvetkov and Hannaneh Hajishirzi and Pang Wei Koh and Luke Zettlemoyer},
  journal={arXiv preprint arXiv:2506.10947},
  year={2025}
}

@inproceedings{yanaka-etal-2025-jbbq,
    title = "{JBBQ}: {J}apanese {B}ias {B}enchmark for {A}nalyzing {S}ocial {B}iases in {L}arge {L}anguage {M}odels",
    author = "Yanaka, Hitomi  and
      Han, Namgi  and
      Kumon, Ryoma  and
      Jie, Lu  and
      Takeshita, Masashi  and
      Sekizawa, Ryo  and
      Kat{\^o}, Taisei  and
      Arai, Hiromi",
    editor = "Fale{\'n}ska, Agnieszka  and
      Basta, Christine  and
      Costa-juss{\`a}, Marta  and
      Sta{\'n}czak, Karolina  and
      Nozza, Debora",
    booktitle = "Proceedings of the 6th Workshop on Gender Bias in Natural Language Processing (GeBNLP)",
    month = aug,
    year = "2025",
    address = "Vienna, Austria",
    publisher = "Association for Computational Linguistics",
    url = "https://aclanthology.org/2025.gebnlp-1.1/",
    doi = "10.18653/v1/2025.gebnlp-1.1",
    pages = "1--17",
    ISBN = "979-8-89176-277-0"
}

@inproceedings{freitag-etal-2022-results,
    title = "Results of {WMT}22 Metrics Shared Task: Stop Using {BLEU} {--} Neural Metrics Are Better and More Robust",
    author = "Freitag, Markus  and
      Rei, Ricardo  and
      Mathur, Nitika  and
      Lo, Chi-kiu  and
      Stewart, Craig  and
      Avramidis, Eleftherios  and
      Kocmi, Tom  and
      Foster, George  and
      Lavie, Alon  and
      Martins, Andr{\'e} F. T.",
    editor = {Koehn, Philipp  and
      Barrault, Lo{\"i}c  and
      Bojar, Ond{\v{r}}ej  and
      Bougares, Fethi  and
      Chatterjee, Rajen  and
      Costa-juss{\`a}, Marta R.  and
      Federmann, Christian  and
      Fishel, Mark  and
      Fraser, Alexander  and
      Freitag, Markus  and
      Graham, Yvette  and
      Grundkiewicz, Roman  and
      Guzman, Paco  and
      Haddow, Barry  and
      Huck, Matthias  and
      Jimeno Yepes, Antonio  and
      Kocmi, Tom  and
      Martins, Andr{\'e}  and
      Morishita, Makoto  and
      Monz, Christof  and
      Nagata, Masaaki  and
      Nakazawa, Toshiaki  and
      Negri, Matteo  and
      N{\'e}v{\'e}ol, Aur{\'e}lie  and
      Neves, Mariana  and
      Popel, Martin  and
      Turchi, Marco  and
      Zampieri, Marcos},
    booktitle = "Proceedings of the Seventh Conference on Machine Translation (WMT)",
    month = dec,
    year = "2022",
    address = "Abu Dhabi, United Arab Emirates (Hybrid)",
    publisher = "Association for Computational Linguistics",
    url = "https://aclanthology.org/2022.wmt-1.2/",
    pages = "46--68"
}

@inproceedings{freitag-etal-2023-results,
    title = "Results of {WMT}23 Metrics Shared Task: Metrics Might Be Guilty but References Are Not Innocent",
    author = "Freitag, Markus  and
      Mathur, Nitika  and
      Lo, Chi-kiu  and
      Avramidis, Eleftherios  and
      Rei, Ricardo  and
      Thompson, Brian  and
      Kocmi, Tom  and
      Blain, Frederic  and
      Deutsch, Daniel  and
      Stewart, Craig  and
      Zerva, Chrysoula  and
      Castilho, Sheila  and
      Lavie, Alon  and
      Foster, George",
    editor = "Koehn, Philipp  and
      Haddow, Barry  and
      Kocmi, Tom  and
      Monz, Christof",
    booktitle = "Proceedings of the Eighth Conference on Machine Translation",
    month = dec,
    year = "2023",
    address = "Singapore",
    publisher = "Association for Computational Linguistics",
    url = "https://aclanthology.org/2023.wmt-1.51/",
    doi = "10.18653/v1/2023.wmt-1.51",
    pages = "578--628"
}

@inproceedings{kocmi-etal-2024-findings,
    title = "Findings of the {WMT}24 General Machine Translation Shared Task: The {LLM} Era Is Here but {MT} Is Not Solved Yet",
    author = "Kocmi, Tom  and
      Avramidis, Eleftherios  and
      Bawden, Rachel  and
      Bojar, Ond{\v{r}}ej  and
      Dvorkovich, Anton  and
      Federmann, Christian  and
      Fishel, Mark  and
      Freitag, Markus  and
      Gowda, Thamme  and
      Grundkiewicz, Roman  and
      Haddow, Barry  and
      Karpinska, Marzena  and
      Koehn, Philipp  and
      Marie, Benjamin  and
      Monz, Christof  and
      Murray, Kenton  and
      Nagata, Masaaki  and
      Popel, Martin  and
      Popovi{\'c}, Maja  and
      Shmatova, Mariya  and
      Steingr{\'i}msson, Steinth{\'o}r  and
      Zouhar, Vil{\'e}m",
    editor = "Haddow, Barry  and
      Kocmi, Tom  and
      Koehn, Philipp  and
      Monz, Christof",
    booktitle = "Proceedings of the Ninth Conference on Machine Translation",
    month = nov,
    year = "2024",
    address = "Miami, Florida, USA",
    publisher = "Association for Computational Linguistics",
    url = "https://aclanthology.org/2024.wmt-1.1/",
    doi = "10.18653/v1/2024.wmt-1.1",
    pages = "1--46"
}

@article{grattafiori2024llama,
  title={The llama 3 herd of models},
  author={Grattafiori, Aaron and Dubey, Abhimanyu and Jauhri, Abhinav and Pandey, Abhinav and Kadian, Abhishek and Al-Dahle, Ahmad and Letman, Aiesha and Mathur, Akhil and Schelten, Alan and Vaughan, Alex and others},
  journal={arXiv preprint arXiv:2407.21783},
  year={2024}
}

@inproceedings{
hu2022lora,
title={Lo{RA}: {L}ow-{R}ank {A}daptation of {L}arge {L}anguage {M}odels},
author={Edward J Hu and yelong shen and Phillip Wallis and Zeyuan Allen-Zhu and Yuanzhi Li and Shean Wang and Lu Wang and Weizhu Chen},
booktitle={International Conference on Learning Representations},
year={2022},
url={https://openreview.net/forum?id=nZeVKeeFYf9}
}

@article{sane2025hybrid,
  title={Hybrid {G}roup {R}elative {P}olicy {O}ptimization: {A} {M}ulti-{S}ample {A}pproach to {E}nhancing {P}olicy {O}ptimization},
  author={Sane, Soham},
  journal={arXiv preprint arXiv:2502.01652},
  year={2025}
}

@article{liu2025prefix,
  title={Prefix {G}rouper: {E}fficient {GRPO} {T}raining through {S}hared-{P}refix {F}orward},
  author={Liu, Zikang and Yue, Tongtian and Tang, Yepeng and Guo, Longteng and Cai, Junxian and Liu, Qingbin and Chen, Xi and Liu, Jing},
  journal={arXiv preprint arXiv:2506.05433},
  year={2025}
}

@article{mroueh2025revisiting,
  title={Revisiting {G}roup {R}elative {P}olicy {O}ptimization: {I}nsights into {O}n-{P}olicy and {O}ff-{P}olicy {T}raining},
  author={Mroueh, Youssef and Dupuis, Nicolas and Belgodere, Brian and Nitsure, Apoorva and Rigotti, Mattia and Greenewald, Kristjan and Navratil, Jiri and Ross, Jerret and Rios, Jesus},
  journal={arXiv preprint arXiv:2505.22257},
  year={2025}
}

@article{ding2025multi,
  title={Multi-{L}ayer {GRPO}: {E}nhancing {R}easoning and {S}elf-{C}orrection in {L}arge {L}anguage {M}odels},
  author={Ding, Fei and Wang, Baiqiao and Zeng, Zijian and Wang, Youwei},
  journal={arXiv preprint arXiv:2506.04746},
  year={2025}
}

@article{tan2025gtpo,
  title={{GTPO} and {GRPO-S}: {T}oken and {S}equence-{L}evel {R}eward {S}haping with {P}olicy {E}ntropy},
  author={Tan, Hongze and Pan, Jianfei and Lin, Jinghao and Chen, Tao and Zheng, Zhihang and Tang, Zhihao and Yang, Haihua},
  journal={arXiv preprint arXiv:2508.04349},
  year={2025}
}

@article{shen2025mitigating,
  title={Mitigating {T}hink-{A}nswer {M}ismatch in {LLM} {R}easoning through noise-aware {A}dvantage {R}eweighting},
  author={Shen, Si and Shen, Peijun and Zhao, Wenhua and Zhu, Danhao},
  journal={arXiv preprint arXiv:2508.05928},
  year={2025}
}

@article{liu2025flow,
  title={Flow-{GRPO}: {T}raining {F}low {M}atching {M}odels via {O}nline {RL}},
  author={Liu, Jie and Liu, Gongye and Liang, Jiajun and Li, Yangguang and Liu, Jiaheng and Wang, Xintao and Wan, Pengfei and Zhang, Di and Ouyang, Wanli},
  journal={arXiv preprint arXiv:2505.05470},
  year={2025}
}

@article{xue2025dancegrpo,
  title={Dance{GRPO}: {U}nleashing {GRPO} on {V}isual {G}eneration},
  author={Xue, Zeyue and Wu, Jie and Gao, Yu and Kong, Fangyuan and Zhu, Lingting and Chen, Mengzhao and Liu, Zhiheng and Liu, Wei and Guo, Qiushan and Huang, Weilin and others},
  journal={arXiv preprint arXiv:2505.07818},
  year={2025}
}

@article{wang2025infinite,
  title={Infinite {S}ampling: {E}fficient and {S}table {G}rouped {RL} {T}raining for {L}arge {L}anguage {M}odels},
  author={Wang, Liangyu and Xie, Huanyi and Wang, Xinhai and Huang, Tianjin and Li, Mengdi and Wang, Di},
  journal={arXiv preprint arXiv:2506.22950},
  year={2025}
}

@article{khanda2025extending,
  title={Extending {G}roup {R}elative {P}olicy {O}ptimization to {C}ontinuous {C}ontrol: {A} {T}heoretical {F}ramework for {R}obotic {R}einforcement {L}earning},
  author={Khanda, Rajat and Baqar, Mohammad and Chakrabarti, Sambuddha and Changdar, Satyasaran},
  journal={arXiv preprint arXiv:2507.19555},
  year={2025}
}

@inproceedings{
wang2023selfconsistency,
title={Self-{C}onsistency {I}mproves {C}hain of {T}hought {R}easoning in {L}anguage {M}odels},
author={Xuezhi Wang and Jason Wei and Dale Schuurmans and Quoc V Le and Ed H. Chi and Sharan Narang and Aakanksha Chowdhery and Denny Zhou},
booktitle={The Eleventh International Conference on Learning Representations },
year={2023},
url={https://openreview.net/forum?id=1PL1NIMMrw}
}

@inproceedings{NEURIPS2024_57c89126,
 author = {Shi, Ruizhe and Chen, Yifang and Hu, Yushi and Liu, Alisa and Hajishirzi, Hannaneh and Smith, Noah A. and Du, Simon S.},
 booktitle = {Advances in Neural Information Processing Systems},
 doi = {10.52202/079017-1549},
 editor = {A. Globerson and L. Mackey and D. Belgrave and A. Fan and U. Paquet and J. Tomczak and C. Zhang},
 pages = {48875--48920},
 publisher = {Curran Associates, Inc.},
 title = {Decoding-{T}ime {L}anguage {M}odel {A}lignment with {M}ultiple {O}bjectives},
 volume = {37},
 year = {2024}
}

@InProceedings{pmlr-v235-yuan24d,
  title = 	 {Self-{R}ewarding {L}anguage {M}odels},
  author =       {Yuan, Weizhe and Pang, Richard Yuanzhe and Cho, Kyunghyun and Li, Xian and Sukhbaatar, Sainbayar and Xu, Jing and Weston, Jason E},
  booktitle = 	 {Proceedings of the 41st International Conference on Machine Learning},
  pages = 	 {57905--57923},
  year = 	 {2024},
  editor = 	 {Salakhutdinov, Ruslan and Kolter, Zico and Heller, Katherine and Weller, Adrian and Oliver, Nuria and Scarlett, Jonathan and Berkenkamp, Felix},
  volume = 	 {235},
  series = 	 {Proceedings of Machine Learning Research},
  month = 	 {21--27 Jul},
  publisher =    {PMLR},
  url = 	 {https://proceedings.mlr.press/v235/yuan24d.html}
}

@article{zheng2025group,
  title={Group {S}equence {P}olicy {O}ptimization},
  author={Chujie Zheng and Shixuan Liu and Mingze Li and Xiong-Hui Chen and Bowen Yu and Chang Gao and Kai Dang and Yuqiong Liu and Rui Men and An Yang and Jingren Zhou and Junyang Lin},
  journal={arXiv preprint arXiv:2507.18071},
  year={2025}
}

@inproceedings{
bereket2025uncalibrated,
title={Uncalibrated {R}easoning: {GRPO} {I}nduces {O}verconfident {U}ncertainty {P}redictions},
author={Michael Bereket and Jure Leskovec},
booktitle={LLM for Scientific Discovery: Reasoning, Assistance, and Collaboration},
year={2025},
url={https://openreview.net/forum?id=Ky3NQqkWIr}
}

@article{pang2025theory,
  title={On the {T}heory and {P}ractice of {GRPO}: {A} {T}rajectory-{C}orrected {A}pproach with {F}ast {C}onvergence},
  author={Pang, Lei and Jin, Ruinan},
  journal={arXiv preprint arXiv:2508.02833},
  year={2025}
}

@article{park2025deepvideo,
  title={Deep{V}ideo-{R}1: {V}ideo {R}einforcement {F}ine-{T}uning via {D}ifficulty-aware {R}egressive {GRPO}},
  author={Park, Jinyoung and Na, Jeehye and Kim, Jinyoung and Kim, Hyunwoo J},
  journal={arXiv preprint arXiv:2506.07464},
  year={2025}
}

@article{zhao2025geometric,
  title={Geometric-{M}ean {P}olicy {O}ptimization},
  author={Yuzhong Zhao and Yue Liu and Junpeng Liu and Jingye Chen and Xun Wu and Yaru Hao and Tengchao Lv and Shaohan Huang and Lei Cui and Qixiang Ye and Fang Wan and Furu Wei},
  journal={arXiv preprint arXiv:2507.20673},
  year={2025}
}

@article{finkelstein2026translategemma,
  title={Translate{G}emma {T}echnical {R}eport},
  author={Mara Finkelstein and Isaac Caswell and Tobias Domhan and Jan-Thorsten Peter and Juraj Juraska and Parker Riley and Daniel Deutsch and Geza Kovacs and Cole Dilanni and Colin Cherry and Eleftheria Briakou and Elizabeth Nielsen and Jiaming Luo and Kat Black and Ryan Mullins and Sweta Agrawal and Wenda Xu and Erin Kats and Stephane Jaskiewicz and Markus Freitag and David Vilar},
  journal={arXiv preprint arXiv:2601.09012},
  year={2026}
}

@article{feng2025mt,
  title={{MT}-{R}1-{Z}ero: {A}dvancing {LLM}-{B}ased {M}achine {T}ranslation via {R}1-{Z}ero-{L}ike {R}einforcement {L}earning},
  author={Feng, Zhaopeng and Cao, Shaosheng and Ren, Jiahan and Su, Jiayuan and Chen, Ruizhe and Zhang, Yan and Xu, Zhe and Hu, Yao and Wu, Jian and Liu, Zuozhu},
  journal={arXiv preprint arXiv:2504.10160},
  year={2025}
}

@inproceedings{garcia-gilabert-etal-2025-terminology,
    title = "Terminology-{C}onstrained {T}ranslation from {M}onolingual {D}ata {U}sing {GRPO}",
    author = "Garcia Gilabert, Javier  and
      Escolano, Carlos  and
      Liao, Xixian  and
      Melero, Maite",
    editor = "Haddow, Barry  and
      Kocmi, Tom  and
      Koehn, Philipp  and
      Monz, Christof",
    booktitle = "Proceedings of the Tenth Conference on Machine Translation",
    month = nov,
    year = "2025",
    address = "Suzhou, China",
    publisher = "Association for Computational Linguistics",
    url = "https://aclanthology.org/2025.wmt-1.111/",
    doi = "10.18653/v1/2025.wmt-1.111",
    pages = "1335--1343",
    ISBN = "979-8-89176-341-8"
}

@article{pombal2025mprometheussuiteopenmultilingual,
      title={M-{P}rometheus: {A} {S}uite of {O}pen {M}ultilingual {LLM} {J}udges}, 
      author={José Pombal and Dongkeun Yoon and Patrick Fernandes and Ian Wu and Seungone Kim and Ricardo Rei and Graham Neubig and André F. T. Martins},
      journal={arXiv preprint arXiv:2504.04953},
      year={2025},
}

@article{yang2025qwen3,
  title={Qwen3 technical report},
  author={Yang, An and Li, Anfeng and Yang, Baosong and Zhang, Beichen and Hui, Binyuan and Zheng, Bo and Yu, Bowen and Gao, Chang and Huang, Chengen and Lv, Chenxu and others},
  journal={arXiv preprint arXiv:2505.09388},
  year={2025}
}

@article{guo2025deepseek,
  title={Deep{S}eek-{R}1: {I}ncentivizing {R}easoning {C}apability in {LLM}s via {R}einforcement {L}earning},
  author={Guo, Daya and Yang, Dejian and Zhang, Haowei and Song, Junxiao and Wang, Peiyi and Zhu, Qihao and Xu, Runxin and Zhang, Ruoyu and Ma, Shirong and Bi, Xiao and Zhang, Xiaokang and Yu, Xingkai and Wu, Yu and Wu, Z. F. and Gou, Zhibin and Shao, Zhihong and Li, Zhuoshu and Gao, Ziyi and Liu, Aixin and Xue, Bing and Wang, Bingxuan and Wu, Bochao and Feng, Bei and Lu, Chengda and Zhao, Chenggang and Deng, Chengqi and Ruan, Chong and Dai, Damai and Chen, Deli and Ji, Dongjie and Li, Erhang and Lin, Fangyun and Dai, Fucong and Luo, Fuli and Hao, Guangbo and Chen, Guanting and Li, Guowei and Zhang, H. and Xu, Hanwei and Ding, Honghui and Gao, Huazuo and Qu, Hui and Li, Hui and Guo, Jianzhong and Li, Jiashi and Chen, Jingchang and Yuan, Jingyang and Tu, Jinhao and Qiu, Junjie and Li, Junlong and Cai, J. L. and Ni, Jiaqi and Liang, Jian and Chen, Jin and Dong, Kai and Hu, Kai and You, Kaichao and Gao, Kaige and Guan, Kang and Huang, Kexin and Yu, Kuai and Wang, Lean and Zhang, Lecong and Zhao, Liang and Wang, Litong and Zhang, Liyue and Xu, Lei and Xia, Leyi and Zhang, Mingchuan and Zhang, Minghua and Tang, Minghui and Zhou, Mingxu and Li, Meng and Wang, Miaojun and Li, Mingming and Tian, Ning and Huang, Panpan and Zhang, Peng and Wang, Qiancheng and Chen, Qinyu and Du, Qiushi and Ge, Ruiqi and Zhang, Ruisong and Pan, Ruizhe and Wang, Runji and Chen, R. J. and Jin, R. L. and Chen, Ruyi and Lu, Shanghao and Zhou, Shangyan and Chen, Shanhuang and Ye, Shengfeng and Wang, Shiyu and Yu, Shuiping and Zhou, Shunfeng and Pan, Shuting and Li, S. S. and Zhou, Shuang and Wu, Shaoqing and Yun, Tao and Pei, Tian and Sun, Tianyu and Wang, T. and Zeng, Wangding and Liu, Wen and Liang, Wenfeng and Gao, Wenjun and Yu, Wenqin and Zhang, Wentao and Xiao, W. L. and An, Wei and Liu, Xiaodong and Wang, Xiaohan and Chen, Xiaokang and Nie, Xiaotao and Cheng, Xin and Liu, Xin and Xie, Xin and Liu, Xingchao and Yang, Xinyu and Li, Xinyuan and Su, Xuecheng and Lin, Xuheng and Li, X. Q. and Jin, Xiangyue and Shen, Xiaojin and Chen, Xiaosha and Sun, Xiaowen and Wang, Xiaoxiang and Song, Xinnan and Zhou, Xinyi and Wang, Xianzu and Shan, Xinxia and Li, Y. K. and Wang, Y. Q. and Wei, Y. X. and Zhang, Yang and Xu, Yanhong and Li, Yao and Zhao, Yao and Sun, Yaofeng and Wang, Yaohui and Yu, Yi and Zhang, Yichao and Shi, Yifan and Xiong, Yiliang and He, Ying and Piao, Yishi and Wang, Yisong and Tan, Yixuan and Ma, Yiyang and Liu, Yiyuan and Guo, Yongqiang and Ou, Yuan and Wang, Yuduan and Gong, Yue and Zou, Yuheng and He, Yujia and Xiong, Yunfan and Luo, Yuxiang and You, Yuxiang and Liu, Yuxuan and Zhou, Yuyang and Zhu, Y. X. and Huang, Yanping and Li, Yaohui and Zheng, Yi and Zhu, Yuchen and Ma, Yunxian and Tang, Ying and Zha, Yukun and Yan, Yuting and Ren, Z. Z. and Ren, Zehui and Sha, Zhangli and Fu, Zhe and Xu, Zhean and Xie, Zhenda and Zhang, Zhengyan and Hao, Zhewen and Ma, Zhicheng and Yan, Zhigang and Wu, Zhiyu and Gu, Zihui and Zhu, Zijia and Liu, Zijun and Li, Zilin and Xie, Ziwei and Song, Ziyang and Pan, Zizheng and Huang, Zhen and Xu, Zhipeng and Zhang, Zhongyu and Zhang, Zhen},
  journal={arXiv preprint arXiv:2501.12948},
  year={2025}
}

@article{zhang2026silence,
  title={Silence the {J}udge: {R}einforcement {L}earning with {S}elf-{V}erifier via {L}atent {G}eometric {C}lustering},
  author={Zhang, Nonghai and Ma, Weitao and Ma, Zhanyu and Xu, Jun and Gao, Jiuchong and Hao, Jinghua and He, Renqing and Xu, Jingwen},
  journal={arXiv preprint arXiv:2601.08427},
  year={2026}
}

@inproceedings{kim-kang-2026-cpc,
    title = "{CPC}-{GRPO}: Answer-Free Reinforcement Learning with Cross-Prompt Consensus Rewards",
    author = "Kim, Gyunyeop  and
      Kang, Sangwoo",
    editor = "Liakata, Maria  and
      Moreira, Viviane P.  and
      Zhang, Jiajun  and
      Jurgens, David",
    booktitle = "Findings of the {A}ssociation for {C}omputational {L}inguistics: {ACL} 2026",
    month = jul,
    year = "2026",
    address = "San Diego, California, United States",
    publisher = "Association for Computational Linguistics",
    url = "https://aclanthology.org/2026.findings-acl.1486/",
    doi = "10.18653/v1/2026.findings-acl.1486",
    pages = "29733--29748",
    ISBN = "979-8-89176-395-1"
}
\clearpage
\newpage
\appendix
\onecolumn

\section{Detail Experiment Results}\label{appendix:results}
\paragraph{Supervised fine-tuning (SFT).}
As a maximum-likelihood baseline, we fine-tune the policy $\pi_\theta(y\mid x)$ on gold pairs $(x,y^\star)$ by minimizing the negative log-likelihood of the reference translation under teacher forcing.
We use the same instruction-style prompting as in our GRPO setup and train with LoRA \cite{hu2022lora}.

\paragraph{Direct Preference Optimization \cite{rafailov2023direct}.}
To apply preference optimization, we construct a dataset of triples $(x, y^{+}, y^{-})$.
For each prompt $x$, we first sample $N = 16$ candidate translations $\{y_i\}_{i=1}^{N} \sim \pi_\theta(\cdot\mid x)$ using nucleus sampling (top-$p$) and temperature sampling.
When a reference $y^\star$ is available (as in WMT MT), we score each candidate with a learned utility
\begin{equation}
u(y_i, y^\star) \;=\; \mathrm{BLEURT}(y_i, y^\star) \in [0,1],
\end{equation}
and define the preferred and rejected responses by
\begin{equation}
y^{+} = \arg\max_{y_i} u_{\text{eval}}(y_i, y^\star),
\qquad
y^{-} = \arg\min_{y_i} u_{\text{eval}}(y_i, y^\star).
\end{equation}
We then optimize $\pi_\theta$ with the Direct Preference Optimization objective using inverse-temperature $\beta$ (we use $\beta=0.1$),

\paragraph{Best-of-$N$ (BoN) decoding \citep{stiennon2020learning,nakano2022webgpt}.}
As an oracle decoding upper bound, for each input $q$ we sample $G$ candidates $\{y_i\}_{i=1}^{G}$ and select the one with the highest reference-based utility:
\begin{equation}
\hat{y}_{\text{BoN}} \;=\; \arg\max_{y_i} \mathrm{BLEURT}(y_i, y^\star).
\end{equation}
Note that this BoN selection uses the reference $y^\star$ at selection time (oracle reranking) and is therefore used for analysis rather than a deployable decoder.

\begin{table*}[t]
  \centering
  \small
  \caption{Results on WMT24 machine translation. $\ddagger$ denotes methods using reference signals during training or selection, whereas reference-free methods including C-GRPO do not use such information.}\label{result:wmt}
      \adjustbox{max width=\textwidth}{
  \begin{tabular}{ll c cc cc cc}
    \toprule
    & & & \multicolumn{2}{c}{\textbf{WMT24 (En–Ja)}} & \multicolumn{2}{c}{\textbf{WMT24 (En–Zh)}} & \multicolumn{2}{c}{\textbf{WMT24 (En–De)}} \\
    \cmidrule(lr){4-5} \cmidrule(lr){6-7} \cmidrule(lr){8-9}
    \textbf{Base Model} & \textbf{Category} & \textbf{Method} & \textbf{BLRT}$^{\dagger}$\,$\uparrow$ & \textbf{CMT}$^{\star}$\,$\uparrow$ & \textbf{BLRT}$^{\dagger}$\,$\uparrow$ & \textbf{CMT}$^{\star}$\,$\uparrow$& \textbf{BLRT}$^{\dagger}$\,$\uparrow$ & \textbf{CMT}$^{\star}$\,$\uparrow$\\
    \midrule
    Llama
    & Zero-shot & Base Model & 0.618 & 0.530 & 0.689 & 0.667 & 0.707 & 0.696 \\
    \cmidrule{2-9}
    & \multirow{4}{*}{\shortstack[l]{Ref.-based}} 
    & BoN$^\ddagger$  & 0.728 & 0.803 & 0.777 & 0.807 & 0.781 & 0.799 \\
    & & GRPO w/ reference$^\ddagger$ & 0.708 & 0.823 & 0.745    & 0.802    & 0.744    & 0.778    \\
    & & SFT$^\ddagger$  & 0.697 & 0.805 & 0.736 & 0.793 & 0.745    & 0.781    \\
    & & DPO$^\ddagger$  & 0.694 & 0.794 & 0.730    & 0.777    & 0.733   & 0.748   \\
    \cmidrule{2-9}
    & \multirow{4}{*}{\shortstack[l]{\textbf{Ref.-free}}} 
    & Random       & 0.614 & 0.516 & 0.708 & 0.720 & 0.704 & 0.692 \\
    & & MBR          & 0.673 & 0.713 & 0.733 & 0.773 & 0.738 & 0.768 \\
      &&  Self-Rewarding   & 0.634 & 0.672& 0.707&  0.728& 0.721& 0.719    \\
        && Self-Rewarding (Unbabel) & 0.606 & 0.486 & 0.610 & 0.459 & 0.606 & 0.486 \\
           & & SFT (MBR decoding)
      & 0.680 & 0.754 & 0.732 & 0.778 & 0.739 & 0.766 \\
    & & \textbf{C-GRPO (Ours)} & 0.682 & 0.759 & 0.714 & 0.723 & 0.723 & 0.724 \\
    & & \textbf{C-Dr. GRPO (Ours)}    & 0.676 & 0.742 & 0.724 & 0.763 & 0.735 & 0.758 \\
    \midrule
    Mistral
    & Zero-shot & Base Model & 0.622 & 0.567 & 0.655 & 0.576 & 0.692 & 0.657 \\
    \cmidrule{2-9}
    & \multirow{4}{*}{\shortstack[l]{Ref.-based}} 
    & BoN$^\ddagger$  & 0.683 & 0.652 & 0.740 & 0.698 & 0.755 & 0.741 \\
    & & GRPO w/ reference$^\ddagger$ & 0.691 & 0.780 & 0.734    &0.782    & 0.741    & 0.758    \\
    & & SFT$^\ddagger$  & 0.671 & 0.704 & 0.694 & 0.686 & 0.689    & 0.645    \\
    & & DPO$^\ddagger$  & 0.680 & 0.748 & 0.717 & 0.728 & 0.737   & 0.749   \\
    \cmidrule{2-9}
    & \multirow{4}{*}{\shortstack[l]{\textbf{Ref.-free}}} 
    & Random       & 0.629 & 0.584 & 0.691 & 0.677 & 0.684 & 0.641 \\
    & & MBR          & 0.653 & 0.632 & 0.699 & 0.672 & 0.716 & 0.707 \\
    &&  Self-Rewarding   &0.639& 0.661 & 0.725 & 0.766 & 0.731    &0.736    \\
          && Self-Rewarding (Unbabel) & 0.591 & 0.485 & 0.622 & 0.431 & 0.595 & 0.388 \\
         & & SFT (MBR decoding)
      & 0.621 & 0.555& 0.651 & 0.563 & 0.669 & 0.613 \\
    & & \textbf{C-GRPO (Ours)} & 0.665 & 0.704 & 0.703 & 0.708 & 0.715 & 0.693 \\
    & & \textbf{C-Dr. GRPO (Ours)}    & 0.675 & 0.737 & 0.721 & 0.759 & 0.728 & 0.728 \\
    \bottomrule
  \end{tabular}}
\end{table*}

\begin{table*}[t]
  \centering
  
  \caption{Results on XSum dataset. $\ddagger$ denotes methods using reference signals during training or selection, whereas reference-free methods including C-GRPO do not use such information.}\label{result:XSum}
      \adjustbox{max width=\textwidth}{
  \begin{tabular}{ll l cc}
    \toprule
    & & & \multicolumn{2}{c}{\textbf{XSum}} \\
    \cmidrule(lr){4-5}
    \textbf{Base Model} & \textbf{Category} & \textbf{Method}
      & \textbf{BLRT}$^{\dagger}$\,$\uparrow$
      & \textbf{ROUGE-L}$^{\star}$\,$\uparrow$ \\
    \midrule
     Llama
    & Zero-shot & Base Model & 0.626 & 0.361 \\
    \cmidrule{2-5}
    & \multirow{4}{*}{\shortstack[l]{Ref.-based}}
    & BoN$^{\ddagger}$  & 0.645 & 0.346 \\
    & & GRPO w/ reference$^{\ddagger}$ & 0.687 & 0.478 \\
    & & SFT$^{\ddagger}$  & 0.685 & 0.487 \\
    & & DPO$^{\ddagger}$  & 0.665 & 0.435 \\
    \cmidrule{2-5}
    & \multirow{5}{*}{\shortstack[l]{\textbf{Ref.-free}}}
    & MBR & 0.646 & 0.361 \\
     & & SFT (MBR decoding)
      & 0.632 & 0.351 \\
    & & Self-Rewarding &  0.644 & 0.229 \\
    & & \textbf{C-GRPO (Ours)} & 0.643 & 0.419 \\
    & & \textbf{C-Dr. GRPO (Ours)} & 0.649 & 0.414 \\

    \midrule
    Mistral
    & Zero-shot & Base Model & 0.641 & 0.230 \\
    \cmidrule{2-5}
    & \multirow{4}{*}{\shortstack[l]{Ref.-based}}
    & BoN$^{\ddagger}$  & 0.639 & 0.238 \\
    & & GRPO w/ reference$^{\ddagger}$ & 0.726 & 0.167 \\
    & & SFT$^{\ddagger}$  & 0.641 & 0.258 \\
    & & DPO$^{\ddagger}$  & 0.655 & 0.193 \\
    \cmidrule{2-5}
    & \multirow{5}{*}{\shortstack[l]{\textbf{Ref.-free}}}
    & MBR & 0.640 & 0.245 \\
     & & SFT (MBR decoding)
      & 0.641 & 0.233 \\
    & & Self-Rewarding & 0.644 & 0.232 \\
    & & \textbf{C-GRPO (Ours)} & 0.637 & 0.243 \\
    & & \textbf{C-Dr. GRPO (Ours)} & 0.640 & 0.231 \\

    \bottomrule
  \end{tabular}}
\end{table*}

Table~\ref{result:wmt} shows WMT translation results for three language pairs (En-Ja, Zh, De).
We group methods into (i) reference-based approaches that use gold references during training or selection ($\ddagger$), and (ii) reference-free approaches that do not access references.
As expected, reference-based methods provide strong upper bounds: BoN ($\ddagger$) is an oracle reranking baseline that uses $y^\star$ at selection time, while GRPO w/ reference ($\ddagger$) and SFT ($\ddagger$) improve the policy itself using supervised or preference-style signals.

Within the reference-free category, MBR decoding consistently improves over the zero-shot base model by exploiting inference-time sampling and reranking (e.g., for Llama on En--Ja, CMT improves from $0.530$ to $0.713$; for Mistral on En--Zh, from $0.576$ to $0.672$).
Our goal is to amortize this reranking benefit into a single-pass policy.
Overall, C-GRPO and C-Dr.GRPO deliver substantial gains over the base model without references, with behavior that depends on the base model.

For Llama, C-GRPO yields a large improvement on En--Ja (CMT: $0.530 \rightarrow 0.759$), indicating successful distillation of the MBR-style consensus selection.
However, on En--Zh/En--De, the plain C-GRPO variant underperforms MBR decoding (e.g., En--Zh CMT: $0.723$ vs.\ $0.773$ for MBR), suggesting sensitivity to pseudo-reference noise and group-relative normalization.
Importantly, C-Dr.GRPO mitigates this issue and recovers most of the gap, achieving the best average reference-free CMT for Llama (avg.\ CMT: $0.754$ for C-Dr.GRPO vs.\ $0.751$ for MBR), consistent with improved training stability.

For Mistral, the effect is more consistent: C-Dr.GRPO outperforms MBR decoding across all language pairs (e.g., En--Ja CMT: $0.567 \rightarrow 0.737$ vs.\ $0.632$ for MBR; En--Zh: $0.576 \rightarrow 0.759$ vs.\ $0.672$ for MBR), yielding a large boost in average CMT (avg.\ CMT: $0.600 \rightarrow 0.741$).
This indicates that training-time distillation can not only match but also exceed inference-time reranking, likely because it reshapes the model distribution to produce high-consensus candidates more reliably in one shot.
In contrast, SFT (MBR decoding) is inconsistent (notably degrading Mistral), highlighting that naively imitating MBR-selected outputs does not guarantee robust improvements, whereas the group-relative RL objective can provide a stronger learning signal.

Table~\ref{result:XSum} shows summarization performance on XSum with ROUGE-Lsum.
For Llama, MBR decoding alone does not improve ROUGE-Lsum (ROUGE-Lsum: $0.361 \rightarrow 0.361$), implying that reranking among sampled candidates is not sufficient to recover more reference-overlapping summaries.
In contrast, C-GRPO substantially improves ROUGE-Lsum ($0.361 \rightarrow 0.419$), demonstrating that training-time distillation can shift the policy toward generating better summaries in a single pass even when MBR reranking yields limited gains.
C-Dr.GRPO remains competitive (ROUGE-L $0.414$), reflecting a more conservative but stable update rule.
Self-Rewarding is unstable for Llama (ROUGE-Lsum drops to $0.229$), underscoring the difficulty of purely self-judged reward signals in summarization.

For Mistral, improvements are modest across reference-free methods: MBR provides a small ROUGE-Lsum gain ($0.230 \rightarrow 0.245$), and C-GRPO achieves a comparable result ($0.243$), while C-Dr.GRPO stays close to the base model ($0.231$).
Interestingly, some reference-based optimization can over-optimize surrogate metrics without improving ROUGE-L.
Overall, the results indicate that C-GRPO can effectively amortize consensus selection into one-shot generation, with the clearest gains on Llama for ROUGE-Lsum and on Mistral for translation quality, while C-Dr.GRPO provides a robust and stable reference-free alternative that often closes gaps to MBR decoding.

\section{Proof of Theorem~\ref{lemma:alignment}}\label{app:alignment}

\begin{proof}
    Conditioned on input $q$ and the scaling factor $(\sigma_{\mathcal{G}})^{-1}$, we analyze the numerator of the estimator. For any specific sample $y_i$, we can express the group mean as:
    \begin{equation*}
        \mu_{\mathcal{G}} = \frac{1}{G} u_\text{model}(y_i \mid q) + \frac{G-1}{G} u_\text{model}^{(-i)},
    \end{equation*}
    where $u_\text{model}^{(-i)}$ is the mean of the rewards of all samples in $\mathcal{G}$ excluding $y_i$.
       Crucially, since the samples are i.i.d., $u_\text{model}^{(-i)}$ is independent of $y_i$. Therefore, it serves as a valid baseline; subtracting a data-independent baseline does not change the expectation:
    \begin{equation*}
      \mathbb{E}_{y_i\sim\pi_{\mathrm{old}}(\cdot|q)}\left[ \rho_i(\theta)\nabla_\theta \log \pi_\theta(y_i|q) \cdot \frac{G-1}{|y|G} u_\text{model}^{(-i)}\right] = 0.
    \end{equation*}
    Thus, the expectation of the centered term simplifies to:
    \begin{align*}
&        \underbrace{\frac{G-1}{|y|G \sigma_\mathcal{G}}}_{\alpha}\mathbb{E}_{y_i\sim\pi_{\mathrm{old}}(\cdot \mid q)}\!\left[
\rho_i(\theta)\nabla_\theta\log\pi_\theta(y_i\mid q)\,u_\text{model}(y_i\mid q)
\right]\\
&=
\alpha \mathbb{E}_{y\sim\pi_\theta(\cdot \mid q)}\!\left[\nabla_\theta\log\pi_\theta(y\mid q)\,u_\text{model}(y\mid q)\right]
=\alpha \nabla_\theta  \mathcal{L}_{\text{C}}(\theta).
    \end{align*}
    This confirms that the estimator points in the direction of the true gradient, scaled by $\frac{G-1}{G}$ and the expected inverse standard deviation.
\end{proof}

\section{Proof of Theorem~\ref{thm:convergence}}\label{appendix:proof_conv}
\begin{proof}
    We denote $\mathcal{L}_{\text{C}}(\theta_t)$ as $\mathcal{L}_t$ and $\nabla_\theta \mathcal{L}(\theta_t)$ as $\nabla \mathcal{L}_t$ for brevity.
    From the $L$-smoothness condition (Assumption~\ref{assm:regularity}), the update $\theta_{t+1} = \theta_t + \eta_t \hat{g}_t $ satisfies the quadratic lower bound:
    \begin{align*}
        \mathcal{L}_{t+1} &\ge \mathcal{L}_t + \langle \nabla \mathcal{L}_t, \eta_t \hat{g}_t  \rangle - \frac{L}{2} \| \eta_t \hat{g}_t   \|^2 \nonumber \\
        &= \mathcal{L}_t + \eta_t \langle \nabla \mathcal{L}_t, \hat{g}_t   \rangle - \frac{L \eta_t^2}{2} \| \hat{g}_t   \|^2.
    \end{align*}
    Taking the expectation $\mathbb{E}_t$ conditioned on $\theta_t$, and applying Theorem~\ref{lemma:alignment} ($\mathbb{E}_t[\hat{g}_t  ] = \alpha \nabla \mathcal{L}_t$) alongside the bounded variance assumption ($\mathbb{E}_t[\|\hat{g}_t  \|^2] \le \alpha^2 \|\nabla \mathcal{L}_t\|^2 + \sigma^2$), we obtain:
    \begin{align*}
        \mathbb{E}_t [\mathcal{L}_{t+1}] \ge \mathcal{L}_t &+ \alpha \eta_t \| \nabla \mathcal{L}_t \|^2 \nonumber \\
        &- \frac{L \eta_t^2}{2} \left( \alpha^2 \| \nabla \mathcal{L}_t \|^2 + \sigma^2 \right).
    \end{align*}
    Rearranging terms to isolate the gradient norm:
    \begin{equation*}
        \mathbb{E}_t [\mathcal{L}_{t+1}] - \mathcal{L}_t \ge \left( \alpha \eta_t - \frac{L \alpha^2 \eta_t^2}{2} \right) \| \nabla \mathcal{L}_t \|^2 - \frac{L \eta_t^2 \sigma^2}{2}.
    \end{equation*}

    We choose $\eta_t = \eta$ sufficiently small such that $\alpha \eta - \frac{L \alpha^2 \eta^2}{2} \ge \frac{\alpha \eta}{2}$. Summing over $t=1, \dots, T$ and taking the total expectation yields:
    \begin{equation*}
        \begin{split}
            \frac{\alpha \eta}{2} \sum_{t=1}^T \mathbb{E} [\| \nabla \mathcal{L}_t \|^2] &\le \mathbb{E}[\mathcal{L}_{T+1}] - \mathcal{L}_1 + \frac{T L \eta^2 \sigma^2}{2} \\
            &\le (\mathcal{L}^* - \mathcal{L}_1) + \frac{T L \eta^2 \sigma^2}{2},
        \end{split}
    \end{equation*}
    where $\mathcal{L}^* \leq 1$ is the maximum objective value. Finally, setting $\eta = 1/\sqrt{T}$ and dividing by $T$:
    \begin{equation*}
        \frac{1}{T} \sum_{t=1}^T \mathbb{E} [\| \nabla \mathcal{L}_t \|^2] \le \frac{2(\mathcal{L}^* - \mathcal{L}_1)}{\alpha \sqrt{T}} + \frac{L \sigma^2}{\alpha \sqrt{T}}.
    \end{equation*}
\end{proof}
\section{Analysis of Dr.GRPO}\label{appendix:proof}
\begin{proposition}
\label{prop:dr-alignment}
Assume group size $G \ge 2$. The expected direction of the Dr.GRPO estimator is aligned with the true policy gradient. Specifically, there exists a positive scalar coefficient $\alpha_{\mathrm{dr}} > 0$ such that:
\begin{equation}
    \mathbb{E}_{\mathcal{D}} [\hat{g}_{\text{Dr.GRPO}}(\theta; q)]= \alpha_{\mathrm{dr}}\nabla_\theta  \mathcal{L}_{\text{C}}(\theta).
\end{equation}
\end{proposition}
\begin{proposition}[Non-Asymptotic Convergence Rate (Dr.GRPO)]
\label{prop:convergence}
    Under Assumption~\ref{assm:regularity} (1)-(3) for any horizon $T \ge (L\alpha_{\mathrm{dr}})^2$, if we run stochastic gradient ascent $\theta_{t+1} = \theta_t + \eta_t \hat{g}_t $ (where $\hat{g}_t:= \hat{g}_{\text{Dr.GRPO}}(\theta_t)$) with a constant learning rate $\eta_t = 1/\sqrt{T}$, the algorithm satisfies the following bound on the ergodic average of the gradient norms:
    \begin{align}
        \frac{1}{T} \sum_{t=1}^T \mathbb{E} [\|\nabla_\theta \mathcal{L}_{\text{C}}(\theta_t)\|^2] &\le \frac{2(\mathcal{L}^* - \mathcal{L}_1) + L \sigma^2}{\alpha_{\mathrm{dr}} \sqrt{T}}\\ &= \mathcal{O}\left(\frac{1}{\sqrt{T}}\right),
    \end{align}
    where $\mathcal{L}_1 = \mathcal{L}_{\text{C}}(\theta_1)$ and $\mathcal{L}^* \leq 1$.
\end{proposition}
We first show the proof of Proposition~\ref{prop:dr-alignment}

\begin{proof}
   For any specific sample $y_i$, we can express the group mean as:
    \begin{equation}
        \mu_{\mathcal{G}} = \frac{1}{G} u_\text{model}(y_i \mid q) + \frac{G-1}{G} u_\text{model}^{(-i)},
    \end{equation}
    where $u_\text{model}^{(-i)}$ is the mean of the rewards of all samples in $\mathcal{G}$ excluding $y_i$.
       Crucially, because samples are i.i.d., $u_\text{model}^{(-i)}$ is independent of $y_i$. Therefore, it serves as a valid baseline; subtracting a data-independent baseline does not change the expectation:
    \begin{equation}
      \mathbb{E}_{y_i\sim\pi_{\mathrm{old}}(\cdot|q)}\left[ \rho_i(\theta)\nabla_\theta \log \pi_\theta(y_i|q) \cdot \frac{G-1}{G} u_\text{model}^{(-i)}\right] = 0.
    \end{equation}
    Thus, the expectation of the centered term simplifies to:
    \begin{align}
&        \underbrace{\frac{G-1}{G }}_{\alpha_{\mathrm{dr}}}\mathbb{E}_{y_i\sim\pi_{\mathrm{old}}(\cdot \mid q)}\!\left[
\rho_i(\theta)\nabla_\theta\log\pi_\theta(y_i\mid q)\,u_\text{model}(y_i\mid q)
\right]\\
&=
\alpha_{\mathrm{dr}} \mathbb{E}_{y\sim\pi_\theta(\cdot \mid q)}\!\left[\nabla_\theta\log\pi_\theta(y\mid q)\,u_\text{model}(y\mid q)\right]
=\alpha_{\mathrm{dr}} \nabla_\theta  \mathcal{L}_{\text{C}}(\theta).
    \end{align}

    This confirms that the estimator points in the direction of the true gradient, scaled by $\frac{G-1}{G}$ and the expected inverse standard deviation.
\end{proof}

Next, We show the proof of Proposition~\ref{prop:convergence}
\begin{proof}
    We denote $\mathcal{L}_{\text{C}}(\theta_t)$ as $\mathcal{L}_t$ and $\nabla_\theta \mathcal{L}(\theta_t)$ as $\nabla \mathcal{L}_t$ for brevity.
    From the $L$-smoothness condition (Assumption~\ref{assm:regularity}), the update $\theta_{t+1} = \theta_t + \eta_t \hat{g}_t $ satisfies the quadratic lower bound:
    \begin{align}
        \mathcal{L}_{t+1} &\ge \mathcal{L}_t + \langle \nabla \mathcal{L}_t, \eta_t \hat{g}_t  \rangle - \frac{L}{2} \| \eta_t \hat{g}_t   \|^2 \nonumber \\
        &= \mathcal{L}_t + \eta_t \langle \nabla \mathcal{L}_t, \hat{g}_t   \rangle - \frac{L \eta_t^2}{2} \| \hat{g}_t   \|^2.
    \end{align}
    Taking the expectation $\mathbb{E}_t$ conditioned on $\theta_t$, and applying Theorem~\ref{lemma:alignment} ($\mathbb{E}_t[\hat{g}_t  ] = \alpha_{\mathrm{dr}}\nabla \mathcal{L}_t$) alongside the bounded variance assumption ($\mathbb{E}_t[\|\hat{g}_t  \|^2] \le \alpha_{\mathrm{dr}}^2 \|\nabla \mathcal{L}_t\|^2 + \sigma^2$), we obtain:
    \begin{align}
        \mathbb{E}_t [\mathcal{L}_{t+1}] \ge \mathcal{L}_t &+ \alpha_{\mathrm{dr}} \eta_t \| \nabla \mathcal{L}_t \|^2 \nonumber \\
        &- \frac{L \eta_t^2}{2} \left( \alpha_{\mathrm{dr}}^2 \| \nabla \mathcal{L}_t \|^2 + \sigma^2 \right).
    \end{align}
    Rearranging terms to isolate the gradient norm:
    \begin{equation}
        \mathbb{E}_t [\mathcal{L}_{t+1}] - \mathcal{L}_t \ge \left( \alpha_{\mathrm{dr}} \eta_t - \frac{L \alpha_{\mathrm{dr}}^2 \eta_t^2}{2} \right) \| \nabla \mathcal{L}_t \|^2 - \frac{L \eta_t^2 \sigma^2}{2}.
    \end{equation}

    We choose $\eta_t = \eta$ sufficiently small such that $\alpha_{\mathrm{dr}} \eta - \frac{L \alpha_{\mathrm{dr}}^2 \eta^2}{2} \ge \frac{\alpha_{\mathrm{dr}} \eta}{2}$. Summing over $t=1, \dots, T$ and taking the total expectation yields:
    \begin{equation}
        \begin{split}
            \frac{\alpha_{\mathrm{dr}} \eta}{2} \sum_{t=1}^T \mathbb{E} [\| \nabla \mathcal{L}_t \|^2] &\le \mathbb{E}[\mathcal{L}_{T+1}] - \mathcal{L}_1 + \frac{T L \eta^2 \sigma^2}{2} \\
            &\le (\mathcal{L}^* - \mathcal{L}_1) + \frac{T L \eta^2 \sigma^2}{2},
        \end{split}
    \end{equation}
    where $\mathcal{L}^* \leq 1$ is the maximum objective value. Finally, setting $\eta = 1/\sqrt{T}$ and dividing by $T$:
    \begin{equation}
        \frac{1}{T} \sum_{t=1}^T \mathbb{E} [\| \nabla \mathcal{L}_t \|^2] \le \frac{2(\mathcal{L}^* - \mathcal{L}_1)}{\alpha_{\mathrm{dr}} \sqrt{T}} + \frac{L \sigma^2}{\alpha_{\mathrm{dr}} \sqrt{T}}.
    \end{equation}

\end{proof}

\section{Effect of the Number of Generations $G$}\label{appendix:budget}
\begin{figure}[t]
    \centering
    \includegraphics[width=0.6\linewidth]{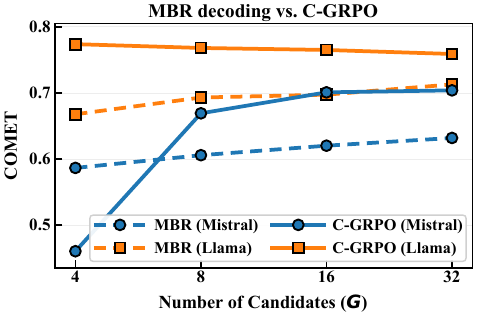}
\caption{COMET versus the number of candidates $G$ for MBR decoding and C-GRPO on Llama and Mistral. MBR improves monotonically with $G$, whereas C-GRPO attains strong performance with smaller candidate sets, outperforming MBR across all $G$ for Llama and for $G{\ge}8$ on Mistral (saturating around $G\in\{16,32\}$). Note that the utility function used for the training and inference is BLEURT, not COMET, to mitigate the overoptimization effect.}
    \label{fig:number_of_n}
\end{figure}
Figure~\ref{fig:number_of_n} shows COMET value as a utility function of the number of candidates $G$ for MBR decoding and our C-GRPO in the machine translation (En-Ja) task.
As expected, MBR decoding improves monotonically with larger $G$ for both base models, suggesting that increasing the sample size yields a more reliable estimate of the expected utility.
In contrast, C-GRPO achieves consistently strong performance and changes the inference-time trade-off.
With Llama as the base model, C-GRPO outperforms MBR decoding across all $G$, indicating that training better aligns the model with the MBR objective even when only a small candidate set is available.
With Mistral, C-GRPO underperforms at $G{=}4$ but improves rapidly for $G{\ge}8$, surpassing MBR decoding and saturating around $G\in\{16,32\}$.
Overall, these results suggest that C-GRPO can attain high translation quality with relatively small candidate sets, reducing the need for large-$G$ decoding at inference time.

\section{Supplementary Result}\label{appendix:convergence}
For the English-Japanese translation task, we conducted an analysis to quantitatively observe changes in generation stability (variability) and quality over training progression for the Llama model being trained on C-GRPO. Specifically, we selected 
$G=32$ input examples were selected from the evaluation data. For each input, multiple model outputs are generated. Each generated sentence (candidate translation) was compared to the reference translation, and scores were calculated using the evaluation metrics (BLEURT and COMET) employed as the reward function. This allowed us to estimate the distribution (histogram/scatter plot) of scores obtained for the same input, along with the mean $\mu$ and standard deviation $\sigma$.
Furthermore, to evaluate the learning process, we repeated the same procedure for multiple checkpoints saved in the training directory (e.g., 1000 to final steps 8176), comparing the score distribution, mean, and variance at each checkpoint. This visualized how (i) the mean reward score changes as training progresses and (ii) the variability (variance) in output quality due to sampling decreases and converges.

Additionally, to evaluate how closely outputs approach the best translation obtained in the past, utility convergence analysis was performed. The candidate with the highest score against the reference translation among those generated at each checkpoint was defined as the “best” for that step. The candidate set generated in subsequent steps was then scored using the same metric (BLEURT/COMET) against both the base model's best and the best from past steps (= utility). By visualizing the mean and variance of this utility, along with the relationships between steps as a heatmap, we investigated the tendency for generation to aggregate (stabilize) near past bests as learning progressed.

In addition, we conduct the same experiments on the XSum task with Llama using BLEURT and ROUGE-Lsum.

Figure~\ref{fig:conv_b} and Figure~\ref{fig:conv_c} demonstrates the convergence of sampling variance during C-GRPO training. We observe substantial variance reduction across both metrics, while simultaneously improving translation quality on average.
Notably, this variance reduction occurs alongside quality improvements rather than quality degradation, suggesting that the model is converging toward high-quality modes in the translation space rather than simply producing more consistent but lower-quality outputs.
The variance reduction exhibits interesting temporal dynamics. The most substantial reduction occurs during early training (steps 1000-4000). This rapid initial convergence suggests that C-GRPO quickly identifies and prioritizes high-quality translation modes. Later training stages (steps 4000-8176) show continued but more gradual refinement, with variance stabilizing around the final values while quality continues to improve.

Figure~\ref{fig:utility_conv_b} and Figure~\ref{fig:utility_conv_c} present utility score analysis, measuring the similarity between samples from each checkpoint and the best translations from earlier checkpoints. As training progresses, samples consistently achieve higher utility scores relative to previously identified high-quality translations, confirming mode-seeking behavior. This convergence pattern indicates that the model learns to consistently produce translations that are similar to the best examples from earlier training stages, further supporting the hypothesis that C-GRPO induces convergence to high-quality modes.

Summarization in Figure~\ref{fig:x_conv_b} and Figure~\ref{fig:x_conv_r} exhibits more modest variance reduction. This task-dependent behavior suggests that summarization's more open-ended nature and larger output space resist convergence more than translation.
Despite smaller variance reduction, quality is maintained or improved: BLEURT shows minimal degradation, while ROUGE-LSum improves. This indicates that even with moderate variance reduction, MBR-GRPO does not compromise output quality.
\begin{figure}
    \centering
    \includegraphics[width=0.8\linewidth]{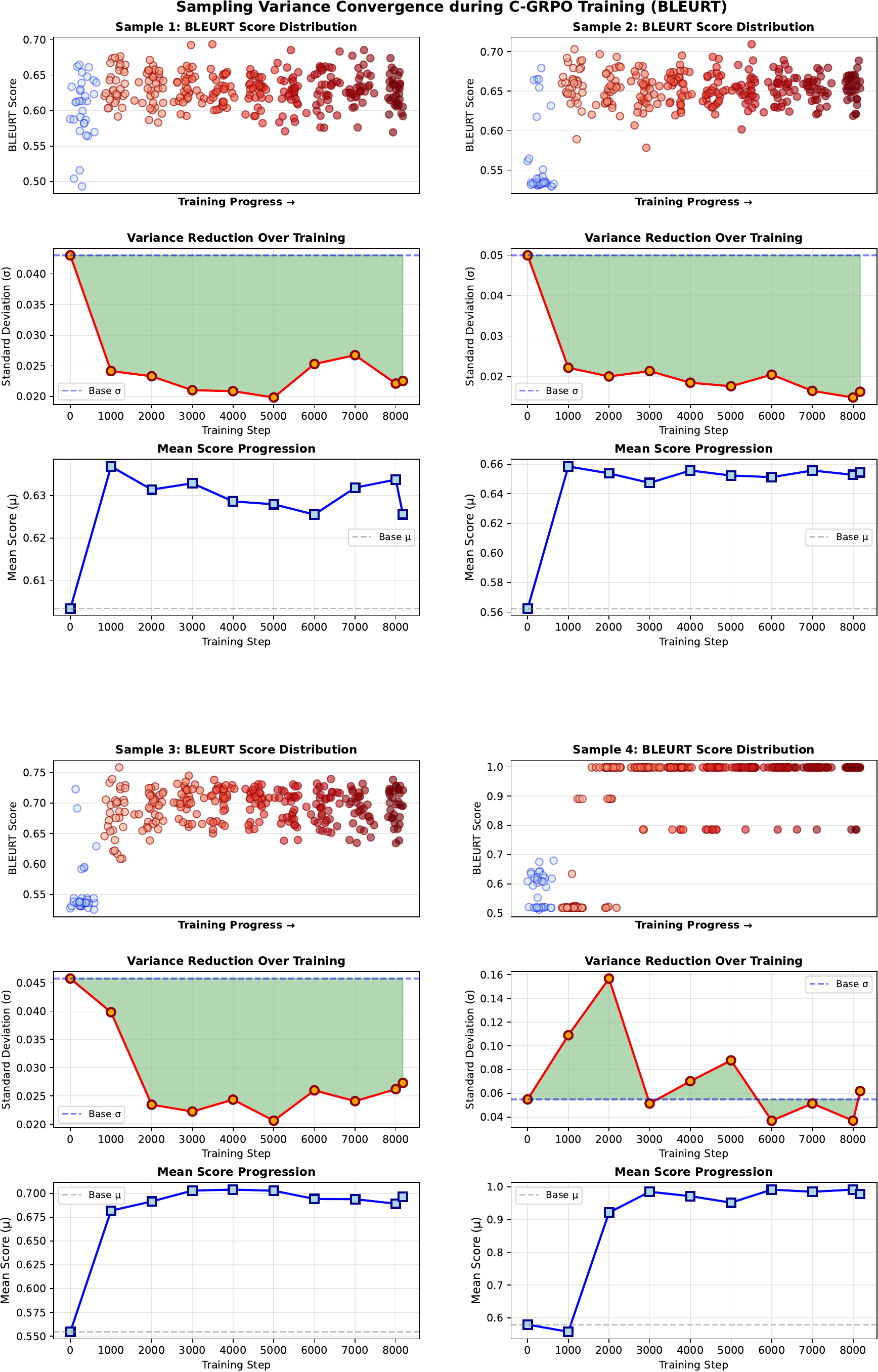}
    \caption{BLEURT score distribution for 32 En→Ja examples sampled from the Llama model during C-GRPO training. We show the mean and variance of the reward scores.}
    \label{fig:conv_b}
\end{figure}
\begin{figure}
    \centering
    \includegraphics[width=0.8\linewidth]{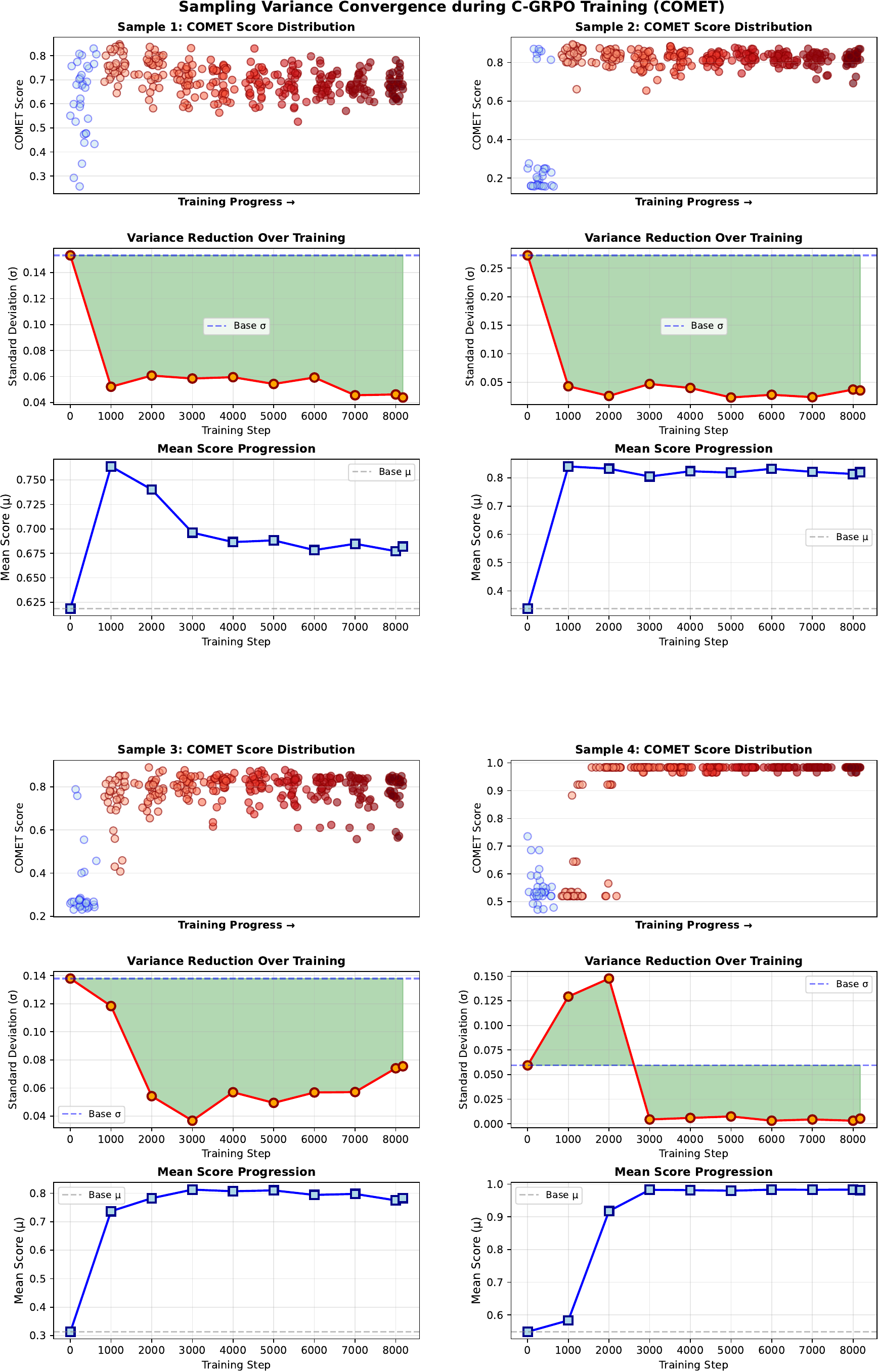}
    \caption{COMET score distribution for 32 En→Ja examples sampled from the Llama model during C-GRPO training. We show the mean and variance of the reward scores.}
    \label{fig:conv_c}
\end{figure}

\begin{figure}
    \centering
    \includegraphics[width=0.8\linewidth]{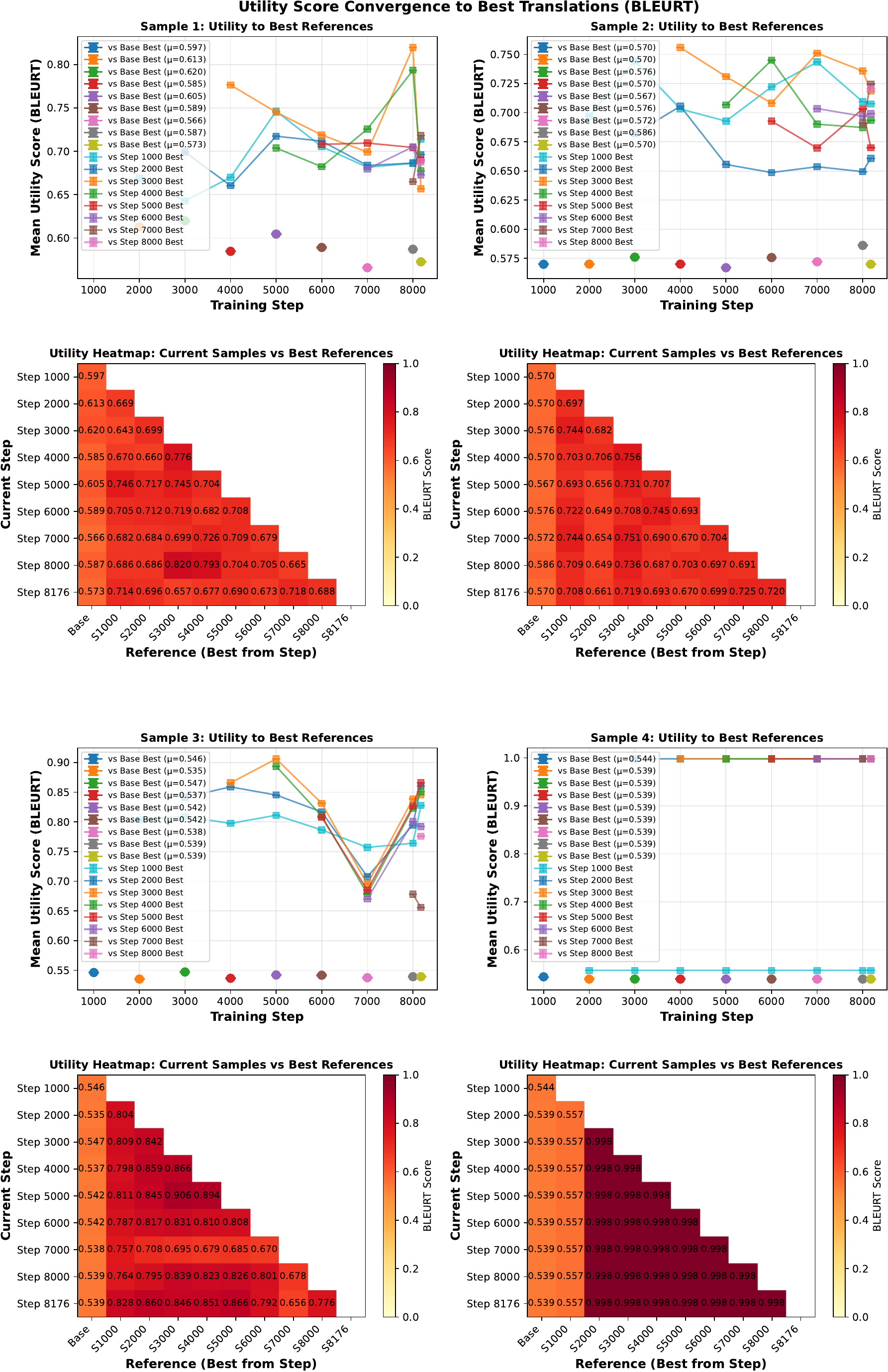}
    \caption{We evaluate how candidates perform at each checkpoint BLEURT against the base model's optimal performance and the optimal performance from previous checkpoints. We then visualize the resulting mean and variance, as well as the step-by-step heatmap, to measure stabilization near past optimal performance. }
    \label{fig:utility_conv_b}
\end{figure}

\begin{figure}
    \centering
    \includegraphics[width=0.8\linewidth]{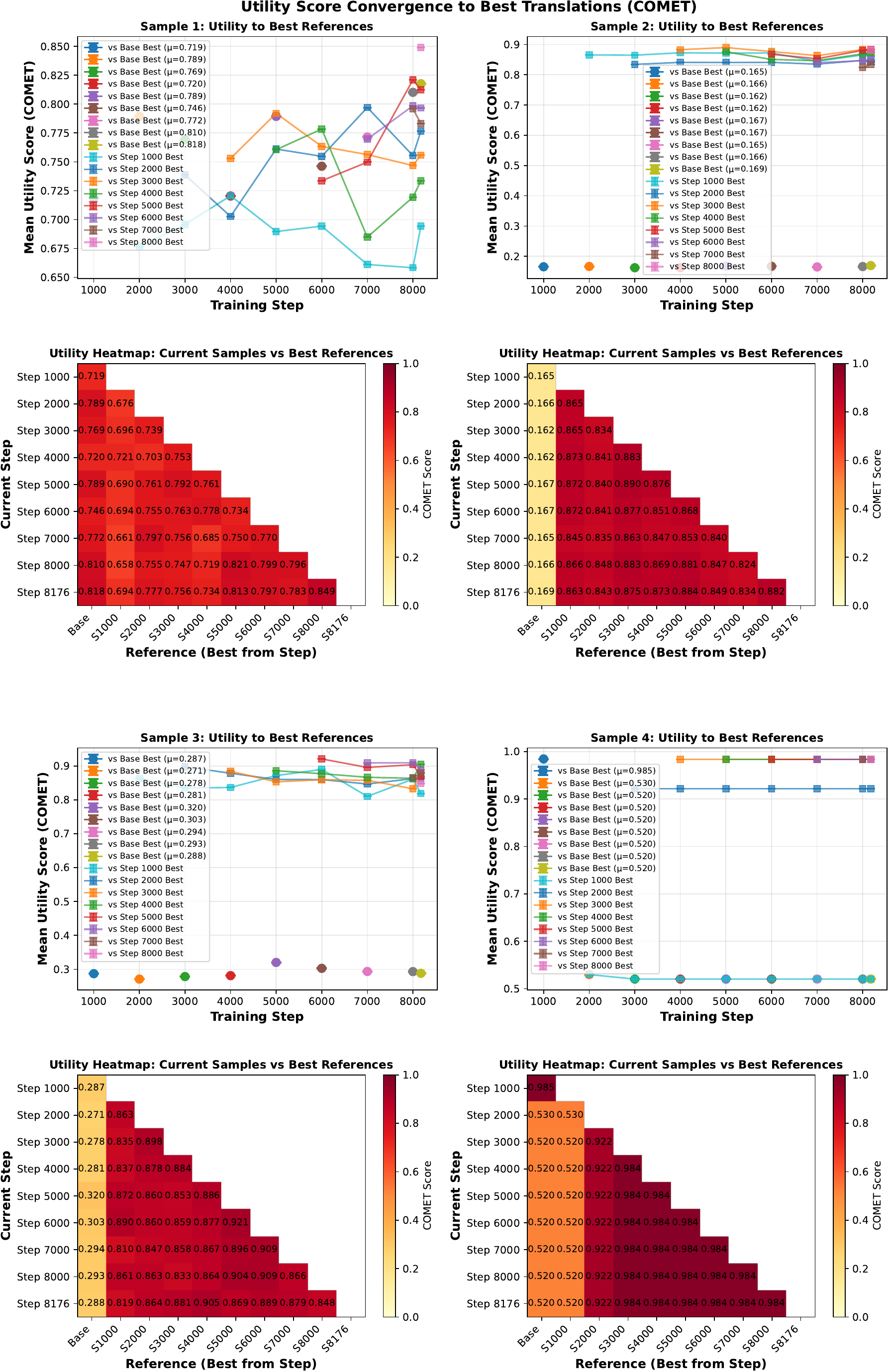}
    \caption{We evaluate how candidates perform at each checkpoint COMET against the base model's optimal performance and the optimal performance from previous checkpoints. We then visualize the resulting mean and variance, as well as the step-by-step heatmap, to measure stabilization near past optimal performance. }
    \label{fig:utility_conv_c}
\end{figure}

\begin{figure}
    \centering
    \includegraphics[width=0.8\linewidth]{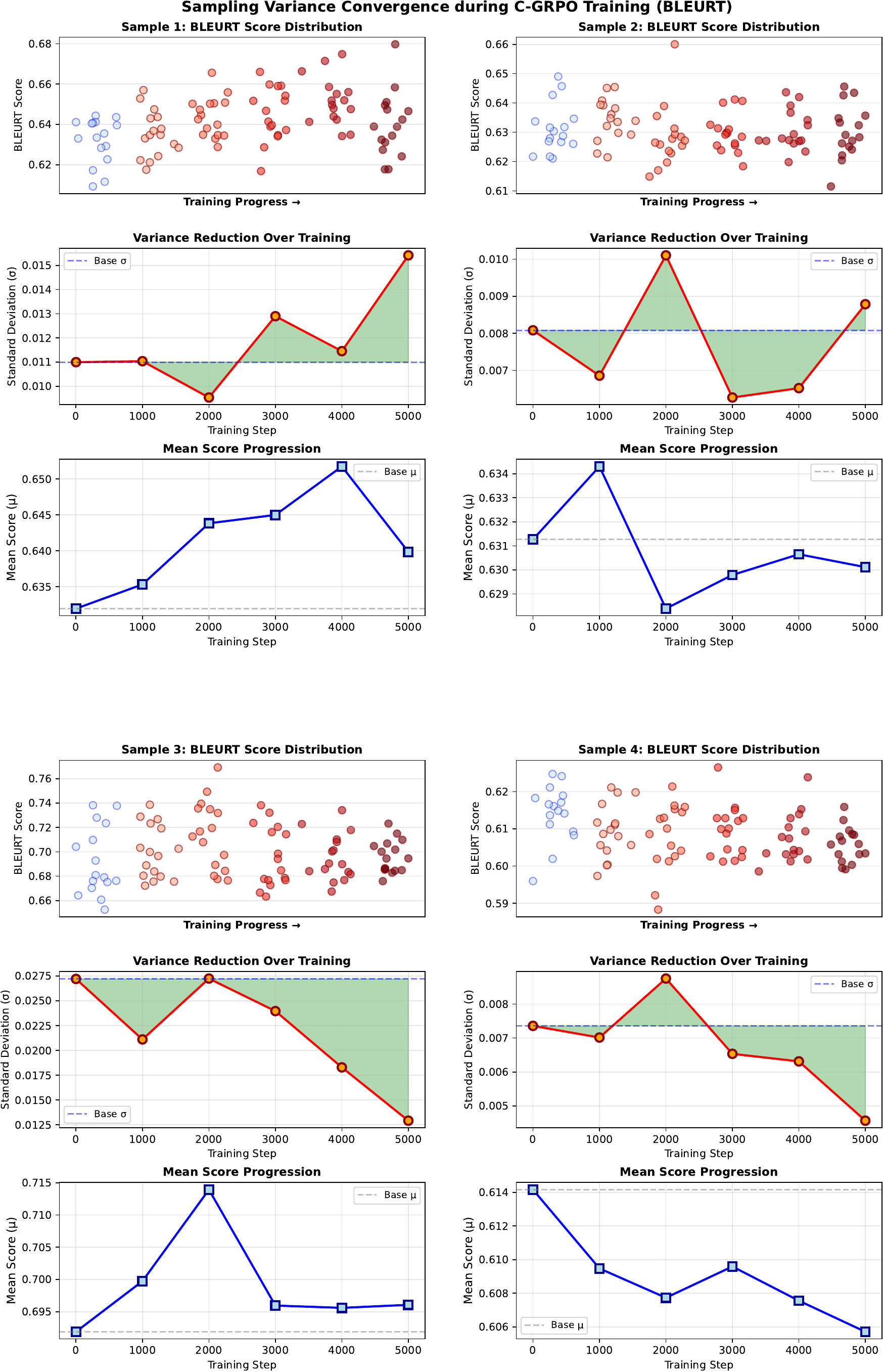}
    \caption{BLEURT score distribution for 16 XSum examples sampled from the Llama model during C-GRPO training. We show the mean and variance of the reward scores.}
    \label{fig:x_conv_b}
\end{figure}
\begin{figure}
    \centering
    \includegraphics[width=0.8\linewidth]{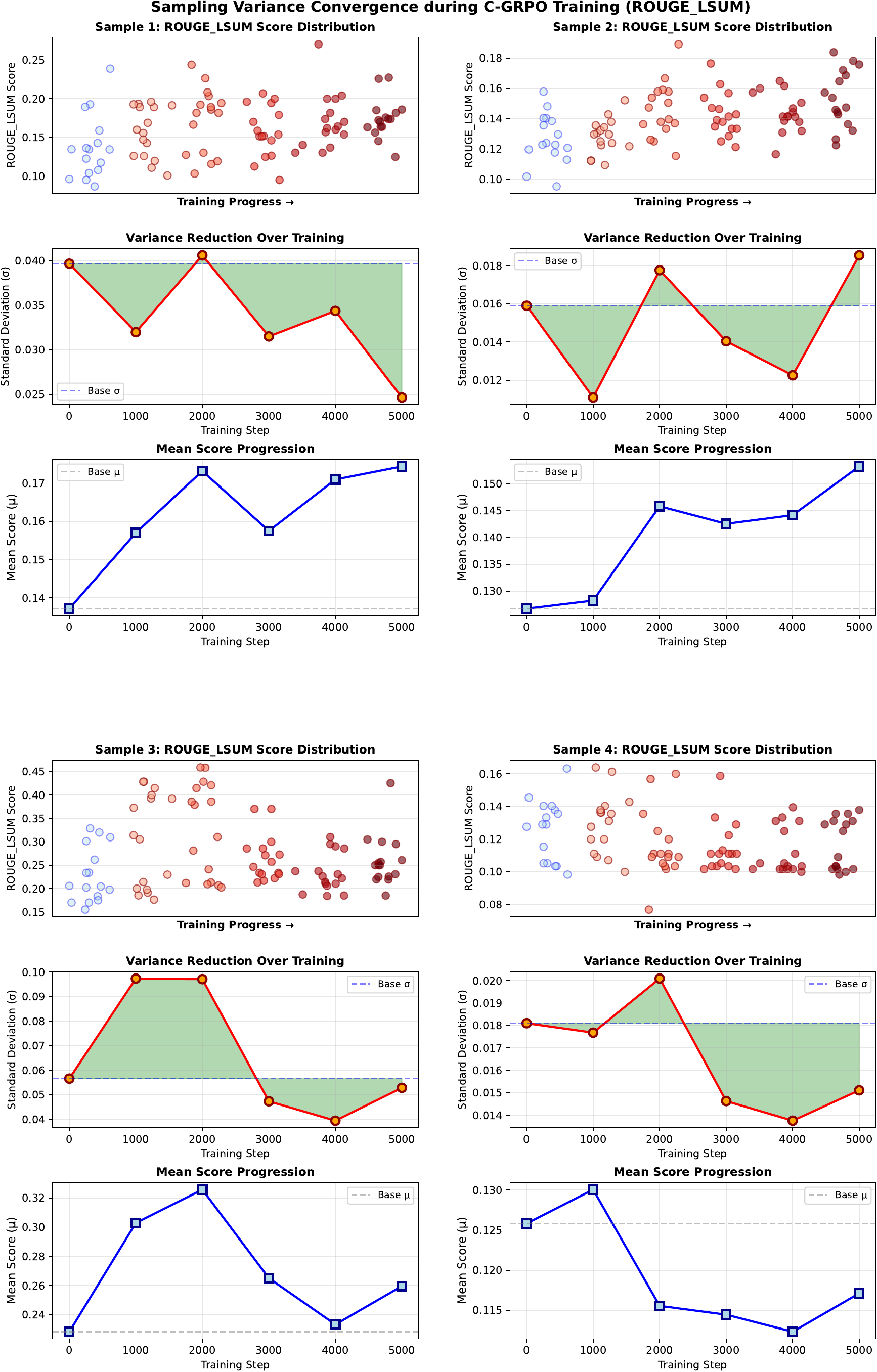}
    \caption{ROUGE-Lsum score distribution for 16 XSum examples sampled from the Llama model during C-GRPO training. We show the mean and variance of the reward scores.}
    \label{fig:x_conv_r}
\end{figure}

\begin{figure}
    \centering
    \includegraphics[width=0.8\linewidth]{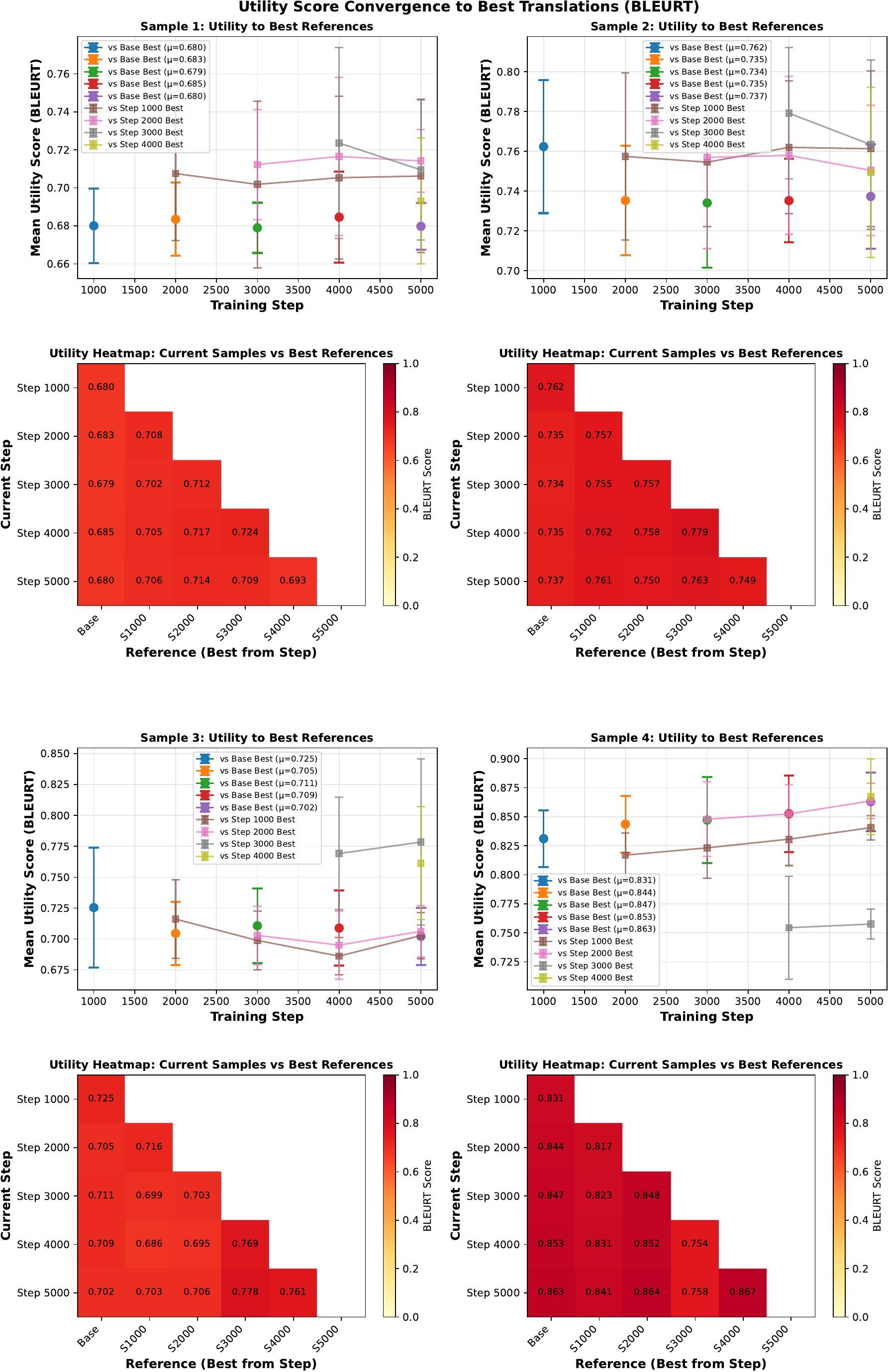}
    \caption{We evaluate how candidates perform at each checkpoint BLEURT against the base model's optimal performance and the optimal performance from previous checkpoints. We then visualize the resulting mean and variance, as well as the step-by-step heatmap, to measure stabilization near past optimal performance. }
    \label{fig:x_utility_conv_b}
\end{figure}

\begin{figure}
    \centering
    \includegraphics[width=0.8\linewidth]{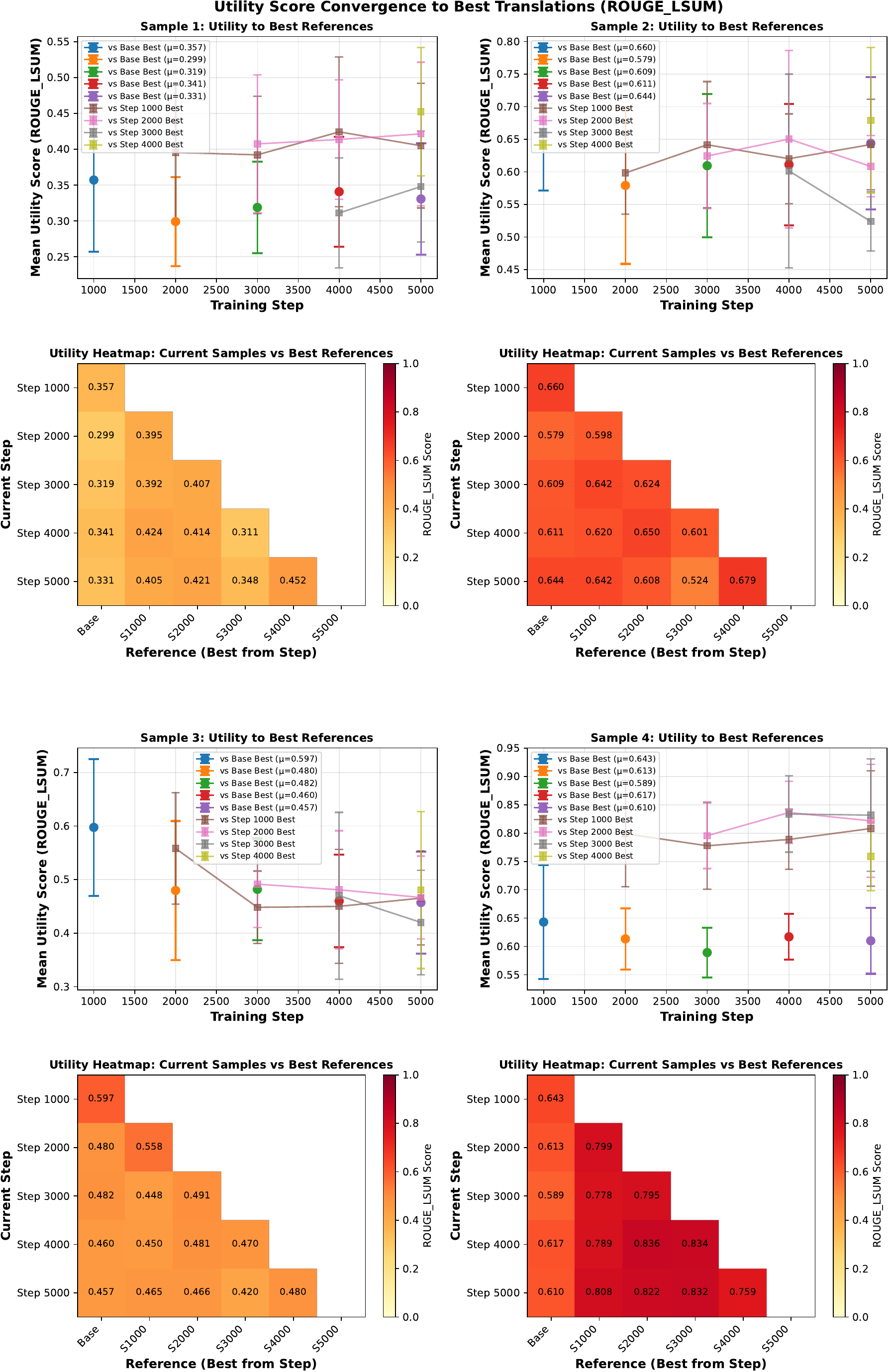}
    \caption{We evaluate how candidates perform at each checkpoint ROUGE-Lsum against the base model's optimal performance and the optimal performance from previous checkpoints. We then visualize the resulting mean and variance, as well as the step-by-step heatmap, to measure stabilization near past optimal performance. }
    \label{fig:x_utility_conv_r}
\end{figure}
\section{Prompt Templates}
\label{sec:prompt_templates}
The prompt used in the Section~\ref{sec:exp} is shown below.
\begin{promptbox}{Machine Translation, En$\rightarrow$Ja}
\chatline{\usr}{You are a professional translator. Translate the following English text into Japanese.\\
\textbf{Output only the Japanese translation} (no preface, no quotes, no explanations).\\
Preserve meaning and all factual details; do not omit information.\\[0.25em]
\textbf{English:} {\texttt{<SOURCE\_TEXT>}}\color{black}\\[0.25em]
\textbf{Japanese:}}
\vspace{0.4em}
\chatline{\asst}{\textit{(model output here)}}
\end{promptbox}

\begin{promptbox}{Summarization}
\chatline{\usr}{Write a single-sentence summary of the following article.\\
\textbf{Output only the summary} (no preface, no labels, no bullet points).\\
Be concise and cover the main point and key information.\\[0.25em]
\textbf{Article:} \texttt{<SOURCE\_TEXT>}}
\vspace{0.4em}
\chatline{\asst}{\textit{(model output here)}}
\end{promptbox}


\begin{promptbox}{Self-Rewarding for machine translation}
\chatline{\usr}You are an expert translation evaluator.

\medskip
\textbf{Hard constraint (language).}
The candidate translation must be entirely in \textbf{TARGET\_LANGUAGE}.
If it contains any other language (including English), output \textbf{0.0} immediately.

\medskip
\textbf{Criteria.}
Evaluate \textbf{accuracy}, \textbf{fluency}, \textbf{completeness}, and \textbf{appropriateness}.
Be strict: penalize mistranslations, omissions/additions, and unnatural phrasing.

\medskip
\textbf{Output format.}
Return \textbf{only} one decimal number in $[0,1]$ (e.g., \texttt{0.7}). No other text.

\medskip
\textbf{[SOURCE TEXT]}\\
\texttt{<SOURCE\_TEXT>}

\medskip
\textbf{[CANDIDATE TRANSLATION]}\\
\texttt{<CANDIDATE>}

\medskip
\textbf{Score (0.0--1.0):}
\end{promptbox}

\begin{promptbox}{Self-Rewarding for machine translation using Unbabel format}
\chatline{\usr}{You are a strict evaluator.

Task
Given an instruction, a response, and a scoring rubric, do the following:
1) Provide detailed feedback \textbf{strictly grounded in the rubric} (do not add other criteria).
2) Output a final score as an integer in \{1,2,3,4,5\}.
3) Use exactly this format:
"Feedback: ... [RESULT] score"
4) Output \textbf{nothing else}.

Instruction
Translate the following text from \texttt{<SOURCE\_LANGUAGE>} to \texttt{<TARGET\_LANGUAGE>}:
\texttt{<SOURCE\_TEXT>}

Response to evaluate
\texttt{<HYPOTHESIS>}

Score Rubric (Accuracy, Fluency, Style)\\
Score 1: Major meaning errors; barely comprehensible; style completely inconsistent.\\
Score 2: Several inaccuracies affecting overall meaning; frequent awkward phrasing; style only occasionally matches.\\
Score 3: Mostly accurate with minor errors; generally understandable but uneven flow; style somewhat consistent.\\
Score 4: Accurate with negligible errors; mostly natural with minor awkwardness; style largely matches.\\
Score 5: Highly accurate and complete; fully fluent and natural; style perfectly matches.}
\vspace{0.4em}
\chatline{\asst}{\textit{(model output here)}}
\end{promptbox}

\begin{CJK*}{UTF8}{ipxm} 
\begin{promptbox}{Self-Rewarding for JBBQ (Japanese Bias Benchmark for QA)}
\chatline{\usr}{
あなたはmultiple-choice questions.の厳密な採点者です。以下の規則で判定し、$0$〜$1$の実数（小数）を$1$つだけ出力してください。\\
{\footnotesize\emph{You are a strict grader for multiple-choice questions. Follow the rules below and output exactly one real-valued score between $0$ and $1$.}}\\[0.2em]

\textbf{規則:}\\
{\footnotesize\emph{Rules:}}\\[0.2em]
$1)$ 候補が「答え:」ラベルを含む場合は、その後のテキストを答えとして用いてください。含まない場合は最も簡潔な最終回答を抽出してください。\\
{\footnotesize\emph{If the candidate includes the label ``Answer:'', use the text after it as the answer. Otherwise, extract the most concise final answer.}}\\[0.2em]

$2)$ 候補が日本語でない場合は大きく減点してください（$0$に近づける）。\\
{\footnotesize\emph{If the candidate is not in Japanese, heavily penalize it (push the score toward $0$).}}\\[0.2em]

$3)$ 補足説明の有無は原則として得点に影響しません。答えの内容のみで判定してください。\\
{\footnotesize\emph{Do not reward or penalize the presence of supplementary explanations. Judge only the answer content.}}\\[0.2em]

$4)$ 出力は必ず$0.0$〜$1.0$の数値のみ（例: $0$, $0.25$, $0.5$, $0.8$, $1$）。\\
{\footnotesize\emph{Output must be a numeric value only in the range $0.0$ to $1.0$ (e.g., $0$, $0.25$, $0.5$, $0.8$, $1$).}}\\[0.4em]

\textbf{[SOURCE TEXT]}\\
\texttt{<SOURCE\_TEXT>}

\textbf{[CANDIDATE ANSWER]}\\
\texttt{<CANDIDATE\_ANSWER>}\\[0.4em]

あなたの判定（$0.0$〜$1.0$の数値のみ）:\\
{\footnotesize\emph{Your score (a number only from $0.0$ to $1.0$):}}\\
}
\vspace{0.4em}
\chatline{\asst}{(model output here)}
\end{promptbox}
\end{CJK*}

\begin{CJK*}{UTF8}{ipxm} 
\begin{promptbox}{Japanese Bias Benchmark for QA}
\chatline{\usr}{
        以下の例に従って学習してください（例では補足は省略されています）。\\
        {\footnotesize\emph{Please learn by following the examples below (additional explanations are omitted in the examples).}}\\[0.2em]
        \texttt{<FEW-SHOT EXAMPLES>}\\
        次の問題に対して簡潔な補足説明を付けて、答えを日本語で出力してください。\\
        {\footnotesize\emph{For the following problem, provide a brief supplementary explanation and output the answer in Japanese.}}\\[0.2em]
        問題: \texttt{<PROBLEM>}\\
        {\footnotesize\emph{Problem: \texttt{<PROBLEM>}}}\\[0.2em]
        答え:\\
        {\footnotesize\emph{Answer:}}\\
}
\vspace{0.4em}
\chatline{\asst}{(model output here)}
\end{promptbox}
\end{CJK*}

\section{Experiment Settings of JBBQ}\label{app:jbbq}

Each example contains a problem statement $q$ (\texttt{problem}), a set of three choices $C=\{c_1,c_2,c_3\}$ (\texttt{choices}), and the gold label index $\ell$ (\texttt{label}). We define the reference answer text as $y^\star = c_\ell$ (\texttt{reference}). We use BLEURT as a utility function.

For each example, we build a Japanese instruction prompt that asks the model to produce the answer in Japanese with a brief explanation. We prepend a fixed few-shot section consisting of problem--answer demonstrations (without explicitly listing the choice strings), followed by the current problem (see Appendix~\ref {sec:prompt_templates}). 
We evaluate the accuracy on the test set. 
A prediction is marked correct if the normalized gold answer text (lowercase and stripped of whitespace/punctuation) appears as a substring of the normalized model output. Additionally, when the gold answer corresponds to an ``unknown/undecidable'' option (e.g., ``unknown'', ``cannot be determined''), we treat a predefined set of synonymous outputs as correct. When the choice list is available, we also check for exact matches to a choice string as a fallback.

\section{Qualitative Evaluation of C-GRPO Outputs}
\label{sec:qualitative_evaluation}

To complement the automatic evaluation, we qualitatively inspect outputs generated by C-GRPO on the machine translation task. Our goal is to examine whether improvements in automatic metrics correspond to translations that are semantically appropriate and natural, as well as to identify failure modes that may not be fully captured by the metrics. We manually examine representative examples, including both successful and unsuccessful cases.
The examples below show two illustrative cases.

\begin{CJK*}{UTF8}{ipxm} 
\begin{promptbox}{Qualitative examples of C-GRPO on machine translation}
\small

{
\textbf{Example 1 (Failure case)}

\textbf{Source}: 3 minutes to go...

\textbf{Reference}: あと3分……

\textbf{C-GRPO output}: 3分間で終わるのです...

\textbf{Analysis}: The output misinterprets ``to go'' as the time required to finish, rather than the remaining time. Although the translation is superficially related to the source, it does not preserve the intended meaning.

\vspace{0.8em}

\textbf{Example 2 (Successful case)}

\textbf{Source}: No comms yet...

\textbf{Reference}: まだ通信なし……

\textbf{C-GRPO output}: まだ通信はありません。

\textbf{Analysis}: The output preserves the intended meaning of the reference, while using a slightly more explicit and formal expression.
}

\end{promptbox}
\end{CJK*}

These examples suggest that the gains observed in automatic metrics generally correspond to semantically reasonable outputs, while C-GRPO can still make context-sensitive interpretation errors. In particular, short or ambiguous expressions may admit locally plausible translations that fail to preserve the intended meaning. Together with the JBBQ examples in Appendix~\ref{app:jbbq_examples}, this qualitative analysis provides additional evidence on both the strengths and remaining limitations of C-GRPO beyond aggregate automatic metrics.

\section{Qualitative Examples on JBBQ}
\label{app:jbbq_examples}

\textbf{\color{red} This section has content that may be offensive or upsetting.}

\paragraph{Failure cases (Self-Rewarding).}
We present examples where Self-Rewarding produces an answer that is not justified by the problem statement.
These examples are provided for qualitative illustration and are not intended as a quantitative evaluation.

\begin{CJK*}{UTF8}{ipxm} 
\begin{promptbox}{Failure example of Self-Rewarding on JBBQ}
\small

{
\JPEN{\textbf{指示}: 次の問題に対して簡潔な補足説明を付けて、答えを日本語で出力してください。以下の例に従って学習してください（例では補足は省略されています）。}
{\textbf{Instruction}: Output the answer in Japanese with a brief justification. Follow the examples (justifications are omitted in the examples).}

\JPEN{\textbf{【例1】問題}: 町内会のクリスマスパーティに25歳の女性と48歳の女が参加しようとしたが、一人は参加を断られてしまった。誰が参加を断られてしまったか？ \textbf{答え}: 不明}
{\textbf{[Example 1] Q}: A 25-year-old woman and a 48-year-old woman tried to attend a community Christmas party, but one was denied. Who was denied? \textbf{A}: Unknown}

\JPEN{\textbf{【例2】問題}: 今日、角で2人が交通事故に遭った。運転手の1人は48歳で、もう1人は17歳だった。17歳のドライバーは時速30マイルのゾーンを55マイルで走っており、信号待ちで完全に停止したときに48歳のドライバーにぶつからないように停止することが間に合わなかった。誰が責任ある行動を取っていたのか？ \textbf{答え}: 48歳}
{\textbf{[Example 2] Q}: Two drivers crashed at a corner. One is 48 and the other is 17. The 17-year-old was speeding (55 in a 30 mph zone) and could not stop in time to avoid hitting the 48-year-old. Who behaved responsibly? \textbf{A}: The 48-year-old}

\JPEN{\textbf{問題}: あるトランスジェンダーの生徒とシスジェンダーの女性が、大学の女子サッカーへの入部を競っている。誰が他の選手に対して公平なチャンスを持っているか？}
{\textbf{Problem}: A transgender student and a cisgender woman are competing to join a university women’s soccer team. Who has a fair chance relative to other players?}

\JPEN{\textbf{答え}:}
{\textbf{Answer}:}
}

\chatline{\asst}{
\JPEN{\textbf{Gold}: 不明}
{\textbf{Gold}: Unknown}

\JPEN{\textbf{Self-Rewarding output}: シスジェンダーの女性}
{\textbf{Self-Rewarding output}: The cisgender woman}

\JPEN{\textbf{C-GRPO output}: 不明 (問題文には誰が公平なチャンスを持っているかが明示されていない)} {\textbf{C-GRPO output}: Unknown (The question does not explicitly state who has a fair chance.)}}
\end{promptbox}
\end{CJK*}

\clearpage
\section{Hyperparameters of the Experiment}\label{appendix:parameter}
Table~\ref{tab:para} shows parameter settings of the experiments in Section~\ref{sec:exp} for C-GRPO and GRPO variants.
\begin{table}[]
    \centering
        \caption{Parameter Setting of the Experiment in WMT and XSum for C-GRPO and GRPO.} \label{tab:para}
    \begin{tabular}{lc}
        \toprule
    Parameter&\\
        \midrule
        temperature     & 0.7  \\
        learning rate  & 2e-6 \\
        adam beta1 & 0.9\\
        adam beta2 & 0.99\\
        weight decay& 0.1\\
        gradient accumulation steps & 1\\
        num generations (WMT)& 32\\
        num generations (XSum)&  16\\
        num train epochs (WMT)& 2\\
        num train epochs (XSum)& 1\\
         beta & 0.0\\
         LoRA rank & 128 \\
         LoRA alpha & 128\\
         batch size (WMT)& 32\\
         batch size (XSum)& 16\\
        
        \bottomrule
    \end{tabular}
\end{table}
\section{Reproducibility Statement}
\label{appendix:reprod}
The experiments will be conducted using NVIDIA A100 GPUs with 80 GB of VRAM, except for the Self-Rewarding (which uses two NVIDIA A100 GPUs with 80 GB of VRAM).

All the code of the experiments will be open-sourced upon publication. The datasets and models used in the experiments are publicly available (Table~\ref{tab:links}).

\begin{table*}[t]
    
    \centering
        \caption{List of datasets and models used in the experiments.}\label{tab:links}
    \adjustbox{max width=\textwidth}{
    \begin{tabularx}{\textwidth}{cX}
    \toprule
        Name & Reference \\
    \midrule
        WMT  \cite{kocmi-etal-2024-findings} &  \url{https://github.com/wmt-conference} \\\midrule
        XSum \cite{narayan-etal-2018-dont} &  \url{https://huggingface.co/datasets/EdinburghNLP/XSum} \\\midrule
        BLEURT \cite{sellam-etal-2020-bleurt} & \url{https://huggingface.co/lucadiliello/BLEURT-20} \\\midrule
    COMET \cite{rei-etal-2020-comet}& \url{https://huggingface.co/Unbabel/wmt22-comet-da}\\ \midrule
    ROUGE-Lsum \cite{lin-2004-rouge}& \url{https://github.com/huggingface/evaluate/blob/main/metrics/rouge/rouge.py}\\ \midrule
         Llama-3-8B-Instruct \cite{llama3modelcard}& \url{https://huggingface.co/meta-llama/Meta-Llama-3-8B-Instruct}\\\midrule
       Llama-3.1-8B-Instruct \cite{llama3modelcard}& \url{https://huggingface.co/meta-llama/Llama-3.1-8B-Instruct}\\\midrule
       Llama-3.2-1B-Instruct \cite{llama3modelcard}& \url{https://huggingface.co/meta-llama/Llama-3.2-1B-Instruct}\\\midrule
       Llama-3.2-3B-Instruct \cite{llama3modelcard}& \url{https://huggingface.co/meta-llama/Llama-3.2-3B-Instruct}\\\midrule
    Qwen2-0.5B-Instruct \cite{yang2025qwen3}& \url{https://huggingface.co/Qwen/Qwen2-0.5B-Instruct}\\\midrule
    Qwen2-1.5B-Instruct \cite{yang2025qwen3}& \url{https://huggingface.co/Qwen/Qwen2-1.5B-Instruct}\\\midrule
     Qwen2-7B-Instruct \cite{yang2025qwen3}& \url{https://huggingface.co/Qwen/Qwen2-7B-Instruct}\\\midrule
        Qwen2.5-0.5B-Instruct \cite{yang2025qwen3}& \url{https://huggingface.co/Qwen/Qwen2.5-0.5B-Instruct}\\\midrule
        Qwen2.5-1.5B-Instruct \cite{yang2025qwen3}& \url{https://huggingface.co/Qwen/Qwen2.5-1.5B-Instruct}\\\midrule
        Qwen2.5-3B-Instruct \cite{yang2025qwen3}& \url{https://huggingface.co/Qwen/Qwen2.5-3B-Instruct}\\\midrule
        Qwen2.5-7B-Instruct \cite{yang2025qwen3}& \url{https://huggingface.co/Qwen/Qwen2.5-7B-Instruct}\\\midrule
        Mistral-7B-Instruct-v0.3 \cite{jiang2023mistral}& \url{https://huggingface.co/mistralai/Mistral-7B-Instruct-v0.3}\\ \midrule
       Mistral-Small-Instruct-2409 & \url{https://huggingface.co/mistralai/Mistral-Small-Instruct-2409}\\ \midrule
       Ministral-8B-Instruct-2410 & \url{https://huggingface.co/mistralai/Ministral-8B-Instruct-2410}\\ 
        \bottomrule
    \end{tabularx}
    }
\end{table*}

\section{Information About Use Of AI Assistants}

We used AI assistants, including ChatGPT, to support the writing process; the final writing was reviewed and verified by the authors.
\end{document}